\documentclass[journal]{IEEEtran}
\usepackage{amsmath,amsfonts}
\usepackage{algorithmic}
\usepackage{algorithm}
\usepackage{array}
\usepackage{textcomp}
\usepackage{stfloats}
\usepackage{url}
\usepackage[colorlinks=false,]{hyperref}
\usepackage{verbatim}
\usepackage{graphicx}
\usepackage{cite}

\usepackage{amssymb}
\usepackage{amsthm}
\usepackage{enumitem}
\usepackage{comment}
\usepackage{thmtools}
\usepackage{import}   %
\usepackage{caption}
\usepackage{subcaption}
\usepackage{siunitx}
\usepackage{color,soul}
\usepackage[capitalize,nameinlink]{cleveref}
\usepackage{svg}
\declaretheorem[name=Theorem,numberwithin=section]{theorem}
\declaretheorem[name=Proposition,numberwithin=section,sibling=theorem]{proposition}
\declaretheorem[name=Lemma,numberwithin=section,sibling=theorem]{lemma}
\declaretheorem[name=Corollary,numberwithin=section,sibling=theorem]{corollary}

\declaretheorem[name=Remark,numberwithin=section,sibling=theorem,qed=\qedsymbol]{remark}
\declaretheorem[name=Definition,numberwithin=section,sibling=theorem]{definition}
\declaretheorem[name=Features,numbered=no]{featuresenv}
\declaretheorem[name=Problem,numberwithin=section]{problemenv}

\newlist{propertylist}{enumerate}{1}
\newlist{feature}{enumerate}{1}
\setlist[propertylist]{label=\roman{propertylisti}, ref=\thedefinition.\roman{propertylisti},noitemsep}
\setlist[feature]{label=\textbf{F\arabic*}, ref=\textbf{F\arabic*},noitemsep}
\crefname{featurei}{Feature}{Features}

\renewcommand{\thedefinition}{\arabic{section}.\arabic{definition}}

\DeclareMathOperator{\Log}{Log}
\DeclareMathOperator{\tr}{tr}

\DeclareMathOperator{\Lop}{L}

\DeclareRobustCommand{\SL}[1][]{%
  \!\;\mathcal{S}%
  \if\relax\detokenize{#1}\relax%
  \else%
    \!\left(#1\right)%
  \fi%
}
\DeclareRobustCommand{\invSL}[1][]{%
  \!\;\mathcal{S}^{-1}%
  \if\relax\detokenize{#1}\relax%
  \else%
    \!\left(#1\right)%
  \fi%
}
\DeclareRobustCommand{\SR}[1][]{%
  \!\;\mathcal{S}_R%
  \if\relax\detokenize{#1}\relax%
  \else%
    \!\left(#1\right)%
  \fi%
}

\begin{document}

\title{Vector Fields for Path Following on Lie Groups with Application in Robot Control}

\author{Felipe Bartelt, Luciano C. A. Pimenta, Weijia Yao, Vinicius M. Gon\c{c}alves
\thanks{}%
}

\markboth{Journal of \LaTeX\ Class Files,~Vol.~14, No.~8, August~2021}%
{Shell \MakeLowercase{\textit{et al.}}: A Sample Article Using IEEEtran.cls for IEEE Journals}

\maketitle

\begin{abstract}
Many robotic systems allow independent control of position and orientation (pose), including omnidirectional aerial vehicles, underwater robots, and manipulator end-effectors. In many applications, these systems must follow a continuous sequence of poses, leading to either trajectory-tracking or path following formulations. Compared to trajectory-tracking, path following offers important practical advantages. In particular, we focus on the problem of path following on Lie groups. Considering the robots as rigid bodies moving in the 3D space, this path-following problem can be posed as a problem of designing guiding vector fields on the matrix Lie group $\mathrm{SE}(3)$. In this paper, we develop a general vector-field framework for path following on connected matrix Lie groups, of which $\mathrm{SE}(3)$ is a prominent special case. The proposed vector field guarantees convergence to a desired parametric curve from almost all initial conditions while ensuring continuous motion along the path. Furthermore, another interesting feature is that, as opposed to previous works, the control input is ``minimal'' in terms of representation and closer to the engineering application (e.g., the body twist in the case $\mathrm{SE}(3)$). After establishing the general case, the framework is then specialized to $\mathrm{SE}(3)$, of special interest in robotics, yielding an efficient algorithm suitable for real-time robotic control. Experiments with a robotic manipulator tracking complex pose paths demonstrate the effectiveness of the approach. An open-source implementation is also provided.
\end{abstract}


\section{Introduction}

\IEEEPARstart{R}{obotic} platforms such as those shown in \Cref{fig:omnidirectionaldrone} increasingly need to execute complex, coordinated motions that demand the control of position and orientation (pose) independently. For example, omnidirectional drones, tilt-rotor hexacopters \cite{kamel2018voliro,HamandiOmni,aboudorra2023modelling}, and omnidirectional underwater vehicles \cite{Sha2025,Wang2026,Jin2015} are fully actuated platforms with six degrees of freedom that may need to follow arbitrary paths in the pose space for a variety of task, like performing inspection in cluttered spaces. Similarly,  the end-effector of a manipulator may need to execute a complex motion in the pose space. In all these scenarios, the robot (in the manipulators' case,  the end-effector) can be seen as a rigid body in the 3D space whose configuration evolves not in Euclidean space but on a curved configuration manifold: the group of rigid-body
motions, $\text{SE}(3)$.
\begin{figure}[t]
    \centering
    \includegraphics[width=.8\columnwidth]{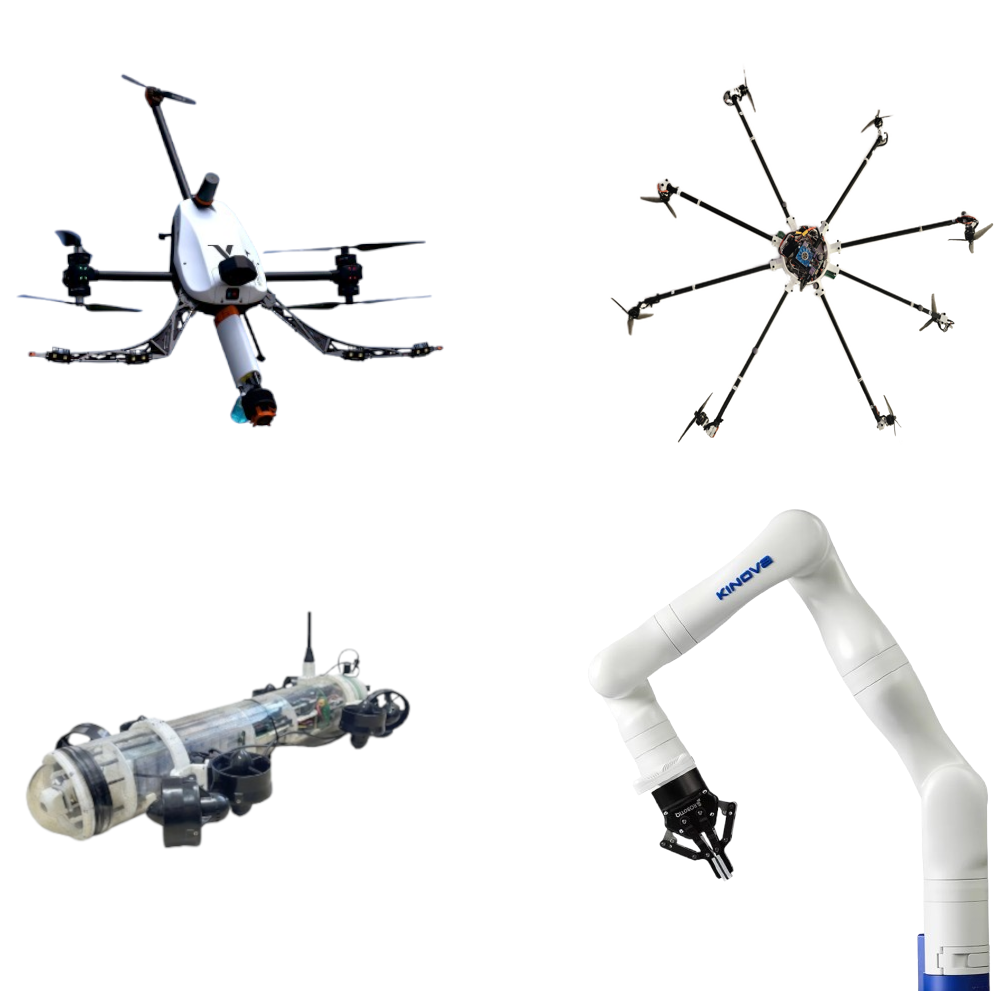}
    \caption{Platforms that benefit from guiding vector fields on $\text{SE}(3)$.
        \textbf{Top left:} a Voliro omnidirectional drone (image from \protect\url{https://voliro.com}).
        \textbf{Top right:} the octorotor design from \cite{HamandiOmni}.
        \textbf{Bottom left:} an omnidirectional underwater vehicle \cite{Sha2025}.
        \textbf{Bottom right:} a Kinova Gen3 manipulator (image from \protect\url{https://www.kinovarobotics.com}).
        All these systems accept six-dimensional velocity commands (twists).}
    \label{fig:omnidirectionaldrone}
\end{figure}

The problem of tracking this continuous sequence of poses can be seen as either a \emph{path following} or \emph{trajectory tracking} problem on the Lie group $\text{SE}(3)$. Path following---guiding a system to a geometric curve without a pre-assigned time parametrization---offers some advantages over trajectory tracking  in many robotic applications (see e.g., \cite{aguiar2008performance,yao2021singularity}). Decoupling the spatial task from temporal specification offers several practical benefits: improved robustness to disturbances, smoother convergence to the path, reduced control effort, and less risk of actuator saturation \cite{Rubi2020}.

In the Euclidean case (i.e., position only), which is the Lie group $\mathbb{R}^n$, guiding vector fields have emerged as an elegant answer to the problem of path following, offering a unified framework that integrates path planning, path following, and control \cite{goncalves2010vector,Rezende2022,yao2021singularity,Gao2022,nunes2023quadcopter,yao2022topological,Chen2025,Yao2023}. In \cite{goncalves2010vector}, a vector field strategy is proposed for robot navigation along curves in the $n$-dimensional Euclidean space. However, this approach depends on implicitly defining target curves as intersections of zero-level surfaces. This method was later improved in \cite{Rezende2022} by using a parametric curve representation on $\mathbb{R}^n$, more convenient for practical implementation.

We build upon that parametric methodology of \cite{Rezende2022} and extend it from the Euclidean space $\mathbb{R}^n$ to \emph{connected matrix Lie groups}. This class includes $\mathbb{R}^n$ and additionally covers configuration spaces that are of great practical interest such as $\text{SO}(3)$ and $\text{SE}(3)$, as well as less-known but also relevant in robotics like the special Galilean group $\text{SGal}(3)$ \cite{kelly2024making,mahony2025galilean}. Our formulation preserves the core convergence-and-traversal structure of \cite{Rezende2022} while now operating intrinsically on the group structure.

The generality of our derivation places it in the landscape of vector-field methods on manifolds and Lie groups. Vector fields on Lie groups have been exploited for learning \cite{Lin2009,urain2022learning}, multi-agent synchronization \cite{Mccarthy2020}, robust attraction to a point \cite{Akhtar2022}, and attitude tracking \cite{Wang2022}. Closest to our work, \cite{yao2022topological} proposes a vector field for curve tracking on a Riemannian manifold embedded in $\mathbb{R}^n$.  Although Riemannian manifolds are more general than Lie groups, our specialization yields a more practical framework in this setting. In our case,  we exploit the Lie group structure to generate control inputs that lie in the Lie algebra, whose dimension $m$ is the intrinsic degrees of freedom of the group (e.g., $3$ for $\text{SO}(3)$). This contrasts with methods that embed the manifold in a higher-dimensional Euclidean space and then project the control back onto the tangent space, (e.g., \cite{yao2022topological}) yielding control inputs in the embedded dimension (e.g., $9$ for $\text{SO}(3)$). Thus, our approach gives control inputs that not only have a smaller dimension, but are also more pratical. For example, in the case of $\text{SE}(3)$, the control input of our controller is (with the appropriate choice of basis for the Lie Algebra), the robots' \emph{twist} (linear and angular velocity vectors). Unlike \cite{yao2022topological}, which defines the target curve via level-set intersections, we adopt the parametric representation used in \cite{Rezende2022}, as it is typically more convenient to specify complex curves in practice.

To extend the vector field construction to general matrix Lie groups, we distilled the essential geometric ingredients that were implicit in \cite{Rezende2022} and become explicit only when the problem is viewed from a broader, group-theoretic perspective. Specifically, the core requirement is a distance function between two group elements---Euclidean distance in \cite{Rezende2022}---that satisfies three properties: \emph{left-invariance} (\Cref{def:distance-left-invariant}), \emph{chainability} (\Cref{def:chainable-distance}), and \emph{local linearity} (\Cref{def:locallylinear}). Once identified, these properties also extend the reach of the method within Euclidean space itself: for instance, any $\ell_p$ norm with $p \geq 2$ now qualifies as a valid distance metric.

We also introduce two elementary operators: the $\Lop$ operator (\Cref{def:Loperator}), analogous to the gradient operator $\nabla$ on a Lie group, and the $\Xi$ operator (\Cref{def:Xioperator}), which extracts the Lie algebra velocity vector and plays a role analogous to the time derivative $\frac{d}{dt}$ on a Lie group. These operators allow us to formulate the controller in a more mathematically concise manner while also making its structure closely resembles that of \cite{Rezende2022}, thereby inheriting some of the intuition behind that work.

The main contributions of this paper are: (1) we extend the vector field strategy from Euclidean space to connected matrix Lie groups, resulting in a non-redundant control input; (2) while providing formal proofs, we derive properties of the distance function that ensure the desired features of the vector field; and (3) we provide an efficient algorithm for implementing this strategy specifically on $\text{SE}(3)$. Although our vector field derivation applies to any connected matrix Lie group, it relies on a distance metric (satisfying the three aforementioned key properties) that we provide explicitly only for what we call \emph{special exponential Lie groups}. This type of group is prevalent in engineering applications (e.g., $\mathbb{R}^n$, $\text{SO}(n)$, $\text{SE}(n)$, $\text{SGal}(3)$).

We also present real-world experiments with a manipulator. A Python implementation of the $\text{SE}(3)$ code, accompanied by a detailed explanation, is available at \url{https://github.com/fbartelt/SE3-example}. 

The remainder of this paper is organized as follows. \cref{sec:preliminaries} introduces the necessary Lie group background and defines the operators used throughout. \cref{sec:vf-lie-group} formulates the guiding vector field problem on matrix Lie groups and establishes convergence guarantees. \cref{sec:simulation} presents experimental results with a robotic manipulator. \cref{sec:conclusion} concludes the paper and discusses future directions.

\section{Preliminaries}\label{sec:preliminaries}
\subsection{Mathematical notation}
We use $\mathbf{I}$ to denote the identity matrix of appropriate dimension, and $\mathbf{e}_k$ to represent the $k^{th}$ column of this identity matrix. The symbol $\mathbf{0}$ is used to denote a zero matrix of appropriate size. We occasionally suppress functional dependencies when the context makes them evident. The symbol $\mathbb{R}_+$ is used to denote the non-negative real numbers. The set of $s$ dimensional vectors $\mathbb{R}^s$ should be considered as a set of $s$ dimensional \emph{column} vectors, i.e., $\mathbb{R}^{s} = \mathbb{R}^{s \times 1}$. The symbol $\exp$ is used for the matrix exponential.

\subsection{Vector field review in Euclidean space}\label{sec:adriano-review}
For clarity, we revisit the vector field strategy presented in \cite{Rezende2022}. Since our work extends this previous approach, this review will help establish direct connections between both works and highlight the core aspects of our contributions. The primary goal of the authors in \cite{Rezende2022} is to develop an artificial $n$-dimensional vector field that guides system trajectories toward a predefined curve and ensures traversal along it. By traversal, we mean that once the system state reaches the curve, it stays there with non-zero velocity. A key element of this formulation is the definition of a distance function with essential properties. As their work focuses solely on Euclidean space, they adopt the Euclidean distance for the vector field computation, which is derived using a parametric representation of the curve.

We summarize the main steps for constructing this vector field, emphasizing the most critical properties. Although the original paper addresses time-varying curves, we limit our discussion to the static case of the methodology. The authors consider a system modeled as a single integrator $\dot{\mathbf{h}} = \boldsymbol{\xi}$, where $\mathbf{h}\in\mathbb{R}^m$ represents the system state, and $\boldsymbol{\xi}\in\mathbb{R}^m$ denotes the system input. The objective is to compute a vector field $\Psi:\mathbb{R}^m\to\mathbb{R}^m$ such that, if the system input is equal to the vector field, the system trajectories converge to and follow the target curve $\mathcal{C}$, for which a parametrization is given by $\mathbf{h}_d(s)$. Despite relying on a parametric representation of the curve, it is important to note that the resulting computations are independent of the specific parametrization chosen.

The authors in \cite{Rezende2022} define their distance function $D$ as the Euclidean distance between the system's current state and the nearest point on the curve, i.e., $D(\mathbf{h}) \triangleq \min_{s}\|\mathbf{h}- \mathbf{h}_d(s)\|$. In this context, we denote by $s^*$ the optimal parameter such that $\mathbf{h}_d(s^*(\mathbf{h}))$ is the closest point on the curve to the current state.
\begin{figure}
    \centering
    \includeinkscape[pretex=\tiny,width=.8\columnwidth]{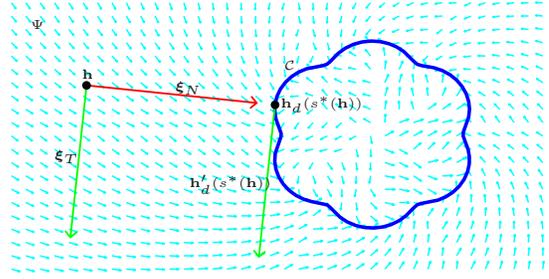}
    \caption{Example showing the vector field and the components for a point $\mathbf{h}\in \mathbb{R}^m$ and curve $\mathcal{C}$.}
    \label{fig:vector-field-adriano}
\end{figure}

Next, the authors in \cite{Rezende2022} introduce two components of their vector field: the \emph{normal} component $\boldsymbol{\xi}_{N}$ (named because it is normal to the curve \(\mathcal{C}\)), responsible for convergence, and the \emph{tangent} component $\boldsymbol{\xi}_{T}$ (named because it is tangent to the curve \(\mathcal{C}\)), which ensures traversal along the curve. The resulting expression for the vector field is $\Psi(\mathbf{h}) = k_N(\mathbf{h})\boldsymbol{\xi}_{N}(\mathbf{h}) + k_T(\mathbf{h})\boldsymbol{\xi}_{T}(\mathbf{h})$, where $k_N$ and $k_T$ are state-dependent gains that weight the contributions of the normal and tangent components. This vector field strategy is illustrated in \Cref{fig:vector-field-adriano}. The normal component, $\boldsymbol{\xi}_{N}$, is naturally taken as the negative transpose\footnote{For consistency, we adopt the gradient of a scalar function $\nabla f$ as a row vector.} of the gradient of the distance function, $\boldsymbol{\xi}_{N} = -(\nabla D)^\top$, due to the use of Euclidean distance.

There is a key aspect of the normal component that is crucial for the convergence proof using Lyapunov stability theory: the fact that the time derivative of the distance function can be expressed as $\dot{D}=-\boldsymbol{\xi}_{N}^{\top}{\boldsymbol{\xi}}$. We emphasize the significance of this feature, as it will play an important role in our extension. In our approach, the normal component is similarly constructed by identifying the term that arises when differentiating the distance function.

Next, we address the tangent component. This component is solely related to the target curve and is defined as the tangent vector at the nearest point on the curve, i.e., $\boldsymbol{\xi}_{T}(\mathbf{h}) = \frac{d}{ds}\mathbf{h}_d(s)|_{s=s^*(\mathbf{h})}$. A noteworthy property of both components is that they are orthogonal to each other, i.e., $\boldsymbol{\xi}_{N}^{\top}{\boldsymbol{\xi}_{T}}=0$, which is essential in the proof of convergence for this algorithm.

In the Lyapunov stability proof, the final result shows that the time derivative of $D$ is negative semidefinite. The proof is then completed using two other essential properties: the fact that the distance function has no local minima outside the curve, and the fact that the gradient of this function never vanishes. With these features, the authors demonstrate that if the system trajectories follow the vector field, the system will converge to and traverse along a predefined curve. A summary of these key features is as follows:
\begin{featuresenv}
\hfill
\begin{feature}
    \item The time derivative of the distance function is the negative of the dot product between the so-called \emph{normal} component and the system control input; \label{feat:adriano-time-derivative-lyapunov-normal-comp}
    \item The \emph{normal} and \emph{tangent} components are orthogonal to each other; \label{feat:adriano-orthogonality}
    \item The distance function has no local minima outside the target curve. Furthermore, whenever the distance function is differentiable, its gradient never vanishes. \label{feat:adriano-no-local-minima}
\end{feature}
\end{featuresenv}
In our generalization, we will incorporate and build upon these features.
\subsection{Lie groups and Lie algebras}
In this section, we recall several properties of Lie groups and Lie algebras that will be essential for developing our extension. First, a Lie group $G$ is a smooth manifold equipped with a group structure. The Lie algebra $\mathfrak{g}$ associated with $G$ has two equivalent definitions, either of which will be used as appropriate. One definition describes the Lie algebra as the tangent space at the group identity $G_e$ \cite[p. 16]{Gallier2020}. Alternatively, a Lie algebra can be viewed as the space of all smooth vector fields on the Lie group, under the Lie bracket operation on vector fields \cite[p. 190]{Lee2012}. Furthermore, a path on a Lie group $G$ is a continuous map $\pi:[0,1]\to G$ that connects two Lie group elements $\pi(0)=\mathbf{H}_0\in G$ and $\pi(1)=\mathbf{H}_1\in G$. We will only consider connected---meaning that for all $\mathbf{H}_0,\mathbf{H}_1\in G$, there exists a path $\pi$ from $\mathbf{H}_0$ to $\mathbf{H}_1$---matrix Lie groups, which are groups where each element allows a matrix representation with a respective inverse. Henceforth we will denote by $n$ the dimension of any (square) matrix in a given group, and the dimension of the group by $m$, which is the same as the dimension of the associated Lie algebra.

Next, we present a fact that will be important for the forthcoming analysis. 
\begin{lemma} \label{lemma:derivative-lie-element-H-parallelizable} Let $\mathbf{G}:\mathbb{R}\to G$ be a differentiable function. Then, there exists a function $\mathbf{A}:\mathbb{R} \to \mathfrak{g}$ such that 
\begin{align*}
    \frac{d}{d\sigma} \mathbf{G}(\sigma) = \mathbf{A}(\sigma) \mathbf{G}(\sigma).
\end{align*}
\end{lemma}
\begin{proof}
    It is a known fact that Lie groups are parallelizable using right-invariant vector fields as a basis \cite{GRIGORIAN2024804}. That means that any vector on the tangent space at any point $\mathbf{G} \in G$ can be written as $\mathbf{AG}$ in which $\mathbf{A}$ is an element of the Lie algebra $\mathfrak{g}$ (possibly different for each $\mathbf{G}$). This implies the desired result.
\end{proof}

 On the other hand, since the Lie algebra is a vector space, we can define a basis $\{\mathbf{E}_k\}\in\mathfrak{g}, k\in\{1,2,\dots, m\}$. Given a basis, for any $\mathbf{A} \in\mathfrak{g}$, there exist scalars $\{\zeta_k\},\, k\in\{1,2,\dots, m\}$, such that $\mathbf{A} = \sum_{k=1}^{m} \zeta_k\mathbf{E}_k$. Furthermore, for each choice of basis for the Lie algebra, it is possible to define a respective linear operator $\SL:\mathbb{R}^{m}\to\mathfrak{g}$ as follows:
\begin{definition}[$\SL$ map]\label{def:SL-left-isomorphism-act-on-xi}
    Let $\boldsymbol{\zeta}$ be an $m$-dimensional vector with entries $\zeta_k$. Given a chosen basis $\mathbf{E}_1,\dots,\mathbf{E}_m$ for an $m$-dimensional Lie algebra $\mathfrak{g}$,  we define the isomorphism $\SL:\mathbb{R}^{m}\to\mathfrak{g}$ as
    \begin{align*}
        \SL[\boldsymbol{\zeta}] \triangleq \sum_{k=1}^m\zeta_k\mathbf{E}_k.    
    \end{align*}
\hfill$\square$\end{definition}
With this, we refer to choosing a basis as a choice of $\SL$ map. Furthermore, the $\SL$ map is an isomorphism between vector spaces, which implies linearity.

Considering \Cref{lemma:derivative-lie-element-H-parallelizable} and \Cref{def:SL-left-isomorphism-act-on-xi}, we can conclude the following important fact.

\begin{lemma} \label{lemma:very-important-fact}
    Given a differentiable function $\mathbf{\mathbf{G}}:\mathbb{R}\to G$ and a chosen $\mathcal{S}$, there exists a function $\hat{\boldsymbol{\zeta}}:\mathbb{R}\to\mathbb{R}^m$, such that
    \begin{align}
    \label{eq:importantresult}
    \frac{d}{d\sigma} \mathbf{\mathbf{G}}(\sigma)=\SL\bigl(\hat{\boldsymbol{\zeta}}(\sigma)\bigr)\mathbf{\mathbf{G}}(\sigma). 
\end{align}

\end{lemma}
\begin{proof} This is a direct consequence of \Cref{lemma:derivative-lie-element-H-parallelizable} and \Cref{def:SL-left-isomorphism-act-on-xi}. 
\end{proof}

Since $\SL$ is an isomorphism, the \emph{inverse map} that maps an element of the Lie algebra to a vector can be defined as well:
\begin{definition}[Inverse $\SL$ map]\label{def:inverse-isomorphism-SLinv}
    Let $\mathfrak{g}$ be an $m$-dimensional Lie algebra. The \emph{inverse map} is defined as $\invSL:\mathfrak{g}\to\mathbb{R}^m$, such that $\invSL[\SL[\boldsymbol{\zeta}]] = \boldsymbol{\zeta}$. 
\hfill$\square$\end{definition}

Now, according to \Cref{lemma:very-important-fact}, for each differentiable function $\mathbf{G}: \mathbb{R} \to G$ there exists a function $\hat{\boldsymbol{\zeta}} : \mathbb{R} \to \mathbb{R}^m$ according to equation \eqref{eq:importantresult}. Thus, we will define the following operator that extracts this $\hat{\boldsymbol{\zeta}}(\sigma)$ from $\mathbf{G}(\sigma)$:
\begin{definition} [$\Xi$ operator] \label{def:Xioperator} Let $G$ be an $m$-dimensional Lie group. Given a choice of $\SL$ map (see \Cref{def:SL-left-isomorphism-act-on-xi}) $\SL: \mathbb{R}^m \to \mathfrak{g}$, the respective $\Xi$ operator maps a differentiable function $\mathbf{G}: \mathbb{R} \to G$ into a function $\Xi[\mathbf{G}]: \mathbb{R} \to \mathbb{R}^m$ as $\Xi[\mathbf{G}](\sigma) = \SL^{-1}\bigl(\frac{d\mathbf{G}}{d\sigma}(\sigma)\mathbf{G}(\sigma)^{-1}\bigr)$, where $\mathbf{G}(\sigma)^{-1}$ denotes the matrix inverse. 
\hfill$\square$\end{definition}

In our development, it will be necessary to take derivatives along the manifold $G$.  For this purpose, the following definition will be useful.
\begin{definition} [$\Lop$ operator] \label{def:Loperator} Let $G$ be an $m$-dimensional Lie group. Given a choice of $S$ map and a differentiable function $f: G \to \mathbb{R}$, we define the $\Lop$ operator such that the function $\Lop[f] : G \to \mathbb{R}^{1 \times m}$ satisfies:
\begin{equation}
\label{eq:Leq}
    \lim_{\epsilon \rightarrow 0} \frac{1}{\epsilon} \Biggl( f\Bigl(\exp\bigl(\SL[\boldsymbol{\zeta}]\epsilon\bigr)\mathbf{G}\Bigr) - f\bigl(\mathbf{G}\bigr) \Biggr) = \Lop[f](\mathbf{G}) \boldsymbol{\zeta}
\end{equation}
for all $\boldsymbol{\zeta} \in \mathbb{R}^m$ and $\mathbf{G}\in G$. Explicitly, the $j^{th}$ entry of the row vector $\Lop[f](\mathbf{G})$ can be constructed as the left-hand side of \eqref{eq:Leq} when $\boldsymbol{\zeta} = \mathbf{e}_j$. In addition, if $f: G \times G \to \mathbb{R}$ is a function of two variables, $f(\mathbf{V},\mathbf{W})$, we define the \emph{partial} $\Lop$ operators $\Lop_{\mathbf{V}}$ and $\Lop_{\mathbf{W}}$ analogously as in \eqref{eq:Leq} but making the variation only in the first or in the second variable, respectively. 
\hfill$\square$\end{definition}

Related to this concept, we define the following:

\begin{definition}[Differentiable function acting on a Lie Group] \label{def:groupdiff}
   We say that a function \(f: G \to \mathbb{R}\) is \emph{differentiable} at a point \(\mathbf{G}=\mathbf{G}_0\) if the limits in \eqref{eq:Leq} for \(\boldsymbol{\zeta} = \mathbf{e}_j\), \(j=1,...,m\)  exist for \(\mathbf{G}=\mathbf{G}_0\) and that the resulting function $\Lop[f]: G \to \mathbb{R}^{1 \times m}$ is continuous at \(\mathbf{G}=\mathbf{G}_0\). Similarly, if \(f: G \times G \to \mathbb{R}\), that is, \(f = f(\mathbf{V},\mathbf{W})\), we define the concepts of being (partially) differentiable with respect to the argument \(\mathbf{V}\) and with respect to the argument  \(\mathbf{W}\). \strut
\hfill$\square$\end{definition}

The $\Lop$ operator is  related to the concept of the \emph{differential} of maps between two manifolds \cite[p. 70]{Lee2012}. If \(\mathcal{U}\) and \(\mathcal{V}\) are two manifolds and \(h: \mathcal{U} \to \mathcal{V}\) is differentiable, its differential at a point \(u \in \mathcal{U}\), the function  \(D_{u}[h]: T_u \mathcal{U} \to T_{h(u)} \mathcal{V}\) (where \(T_p \mathcal{M}\) is the tangent space to the manifold \(\mathcal{M}\) at the point \(p\)), is a linear map defined as:
\begin{equation}
\label{eq:defdiffmanif}
    D_u[h](\mu) = \frac{d}{d\epsilon} h(\gamma(\epsilon)) \Big|_{\epsilon=0}
\end{equation}
\noindent in which \(\mu \in T_u \mathcal{U}\), and \(\gamma: \mathbb{R} \to \mathcal{U}\) is any differentiable curve with \(\gamma(0)=u\), \(\gamma'(0)=\mu\). Once we recognize that the left-hand side of \eqref{eq:Leq} and the right-hand side of \eqref{eq:defdiffmanif} are the same with \(h=f\), \(\mathcal{U} = G\), \(\mathcal{V}=\mathbb{R}\), \(\gamma(\epsilon) = \exp(\SL(\boldsymbol{\zeta})\epsilon)\mathbf{G}\), \(u = \mathbf{G}\) and \(\mu = \SL(\boldsymbol{\zeta})\mathbf{G}\) we see that, for all $\boldsymbol{\zeta} \in \mathbb{R}^m$:
\begin{equation}
\label{eq:importantobs}
    D_{\mathbf{G}}[f]\big(\SL(\boldsymbol{\zeta})\mathbf{G}\big) = \Lop[f](\mathbf{G})\boldsymbol{\zeta}.
\end{equation}
This observation will be useful in the following result, which is a version of the chain rule using the $\Lop$ operator.

\begin{lemma}\label{lemma:chainrule}
    Let $G$ be an $m$-dimensional Lie group. Let $\mathbf{G} : \mathbb{R} \to G$ and $f: G \to \mathbb{R}$ be differentiable functions. Then:
    \begin{align*}
       \frac{d}{d\sigma} f\bigl(\mathbf{G}(\sigma)\bigr) = \Lop[f]\bigl(\mathbf{G}(\sigma)\bigr)\Xi[\mathbf{G}](\sigma).
    \end{align*}
\end{lemma}
\begin{proof}
    This is a consequence of the \emph{chain rule for manifolds \cite[p. 55]{Lee2012}}, which can be stated as follows: let  \(\mathbf{G}: \mathcal{W} \to \mathcal{U}\) and \(f: \mathcal{U} \to \mathcal{V}\), both functions being differentiable. If \(f\circ \mathbf{G} : \mathcal{W} \to \mathcal{V}\) is the composition, for any \(w \in \mathcal{W}\) and any  \(\omega \in \mathcal{T}_w \mathcal{W}\)
    \begin{equation}
    \label{eq:chainrule}
        D_w[f \circ \mathbf{G}](\omega) = D_{\mathbf{G}(w)}[f]\Big(D_w[\mathbf{G}](\omega)\Big).
    \end{equation}
   In this result, the definition of the operator \(D\) is as in \eqref{eq:defdiffmanif}.
    
    Now, it is easy to see from the definition of a differential on manifolds, as given in \eqref{eq:defdiffmanif}, that in the special case where the domain is the manifold \(\mathbb{R}\), that is, \(r: \mathbb{R} \to \mathcal{X}\),
    \begin{equation}
    \label{eq:diffspc}
        \frac{d }{d\sigma} r(\sigma) = \frac{d}{d\epsilon}r(\underbrace{\sigma+\epsilon}_{\gamma(\epsilon)})|_{\epsilon=0} = D_{\sigma}[r](1),
    \end{equation}
    in which we use \eqref{eq:defdiffmanif} and the fact that $\gamma(0)=\sigma$ and $\gamma'(0)=1$.
    
    We finish by taking \(\mathcal{U} = G\), \(\mathcal{V} = \mathcal{W} = \mathbb{R}\), \(w=\sigma\) and \(\omega=1\). Let \(\hat{\boldsymbol{\zeta}}(\sigma) = \Xi[\mathbf{G}](\sigma)\). Applying the chain rule in \eqref{eq:chainrule} and the result in \eqref{eq:diffspc} for both functions \(\mathbf{G}: \mathbb{R} \to  G\) and \( f \circ \mathbf{G} : \mathbb{R} \to \mathbb{R}\):
    \begin{align*}
     \frac{d}{d\sigma} f(\mathbf{G}(\sigma)) {\stackrel{eq. \eqref{eq:diffspc}}{=}} D_{\sigma}[f \circ \mathbf{G}](1) {\stackrel{eq. \eqref{eq:chainrule}}{=}}
         D_{\mathbf{G}(\sigma)}[f]\Big(D_{\sigma}[\mathbf{G}](1) \Big)
         & \\
      \stackrel{eq. \eqref{eq:diffspc}}{=} D_{\mathbf{G}(\sigma)}[f]\Big( \frac{d}{d\sigma }\mathbf{G}(\sigma) \Big) \stackrel{eq. \eqref{eq:importantresult}}{=} D_{\mathbf{G}(\sigma)}[f]\Big(S(\hat{\boldsymbol{\zeta}}(\sigma))\mathbf{G}\Big).
      &
    \end{align*}
    The result finally follows from \eqref{eq:importantobs} once we recall that $\hat{\boldsymbol{\zeta}}(\sigma) = \Xi[\mathbf{G}](\sigma)$. 
\end{proof}

As a corollary of \Cref{lemma:chainrule}:

\begin{corollary} \label{corol:corol1} If we have a function $f: G \times G \to \mathbb{R}$ instead of a function of a single variable, and two differentiable $\mathbf{V}, \mathbf{W} : \mathbb{R} \to G$, then:
\begin{align*}
        \frac{d}{d \sigma} f\bigl(\mathbf{V}(\sigma){,}\mathbf{W}(\sigma)\bigr) {=}
        &
        \Lop_{\mathbf{V}}[f]\bigl(\mathbf{V}(\sigma){,}\mathbf{W}(\sigma)\bigr) \Xi[\mathbf{V}](\sigma) \\
        &
        {+} \Lop_{\mathbf{W}}[f]\bigl(\mathbf{V}(\sigma){,}\mathbf{W}(\sigma)\bigr) \Xi[\mathbf{W}](\sigma).
\end{align*}
\end{corollary}

\section{Vector Field in Matrix Lie Groups}\label{sec:vf-lie-group}
We will now state our problem formally.
\begin{problemenv}\label{problem:curve-follow}
    Consider an $m$-dimensional connected matrix Lie group $G$ and a fully‑actuated, first‑order system
    \begin{align}
        \dot{\mathbf{H}}(t)=\SL\bigl(\boldsymbol{\xi}(t)\bigr)\mathbf{H}(t),
        \label{eq:derivative-H-SL-considered-system}
    \end{align}
    where $\mathbf{H}(t) \in G$ is the state, $\boldsymbol{\xi}(t) \in \mathbb{R}^m$ is the control input, and $\SL: \mathbb{R}^m \to \mathfrak{g}$ is the linear map induced by a chosen basis of the Lie algebra $\mathfrak{g}$ (\Cref{def:SL-left-isomorphism-act-on-xi}).
    Let $\mathcal{C} \subset G$ be a desired curve, defined as the image of a differentiable, non‑self‑intersecting parametrization $\mathbf{H}_d : [0,1] \to G$, and assume a distance function $D:G \to\mathbb{R}_+$ between an element $\mathbf{H} \in G$ and $\mathcal{C}$ is available (e.g. from \cref{def:distance-D-element-curve}).
    
    Design a continuous state-feedback law $\boldsymbol{\xi} = \Psi(\mathbf{H})$, such that the closed‑loop system satisfies
    \begin{enumerate}
        \item \textbf{Convergence:} the distance $D(\mathbf{H}(t))$ tends to zero as $t\to\infty$ for almost all initial conditions, and all remaining trajectories are attracted to a finite set that can be escaped in finite time.
        \item \textbf{Traversal:} once the state reaches $\mathcal{C}$, it stays on the curve and moves along it in the direction of increasing parameter $s$ with strictly positive speed.
    \end{enumerate}\hfill\qedsymbol
\end{problemenv}
\Cref{problem:curve-follow} formalizes the guidance task considered throughout this paper. The dynamics \eqref{eq:derivative-H-SL-considered-system} describe a fully‑actuated, first‑order system evolving on the Lie group $G$. In the special case $G=\text{SE}(3)$, the state $\mathbf{H}$ encodes the pose (position and orientation) of a rigid body, and, with the proper choice of operator $\SL$, the control input $\boldsymbol{\xi} \in \mathbb{R}^6$ corresponds to the body's twist---the stacked linear and angular velocities expressed in the world frame. We shall refer to $\boldsymbol{\xi}$ generically as the \emph{generalized twist}.

\begin{remark}
The results presented in this paper also apply to systems of the form $\dot{\mathbf{H}} = \mathbf{H} \mathcal{S}'(\boldsymbol{\xi}')$, where $\mathcal{S}'$ is an appropriate linear operator mapping $\mathbb{R}^m$ to $\mathfrak{g}$. For instance, considering the UAV example, this system could model a UAV controlled via the twist in the \emph{body frame} rather than the \emph{fixed frame}. This adaptation can be achieved by rewriting \eqref{eq:derivative-H-SL-considered-system} as $\dot{\mathbf{H}} = \mathbf{H}(\mathbf{H}^{-1} \SL[\boldsymbol{\xi}] \mathbf{H})$. It is well known in the context of Lie groups (see \cite[p. 153]{Lee2012}) that for any $\mathbf{A} \in \mathfrak{g}$ and $\mathbf{H} \in G$, the term $\mathbf{H}^{-1} \mathbf{A} \mathbf{H}$ also belongs to the Lie algebra $\mathfrak{g}$. Therefore, there exists a $\boldsymbol{\xi}' \in \mathbb{R}^m$ such that $\mathcal{S}'(\boldsymbol{\xi}') = \mathbf{H}^{-1} \SL[\boldsymbol{\xi}] \mathbf{H}$, since $\mathcal{S}'$ is a bijection. This correspondence enables the calculation of a controller for the modified system based on the controller designed for the original system.
\end{remark}

To solve \cref{problem:curve-follow}, we propose a vector field of the form
\begin{align}
    \Psi\left(\mathbf{H}\right) \triangleq k_N(\mathbf{H})\boldsymbol{\xi}_N(\mathbf{H}) + k_T(\mathbf{H})\boldsymbol{\xi}_T(\mathbf{H}), \label{eq:vector-field-proposition}
\end{align}
where $k_N: G \to \mathbb{R}$ and $k_T: G \to \mathbb{R}$ are continuous functions, and $k_T(\mathbf{H})$ is positive, $k_N(\mathbf{H})=0$ when $\mathbf{H} \in \mathcal{C}$ and $k_N(\mathbf{H}) > 0$ otherwise. The \emph{normal} component $\boldsymbol{\xi}_N$ ensures convergence, and the \emph{tangent} component $\boldsymbol{\xi}_T$ governs traversal along the curve. These components will be formally defined later.

The vector field presented in \cite{Rezende2022} is a special case of our proposed formulation. Since this particular case is easier to visualize, we will instantiate the definitions and results from this section  to aid in the visualization and understanding of our proposed concept later in \Cref{subs:partcase}.

For the vector field computation, we need to measure the distance between an element and a curve within the Lie group. Thus, we first define a distance function $\widehat{D}$ between arbitrary elements $\mathbf{V}$ and $\mathbf{W}$ in the group, as follows.

\begin{definition}[EE-distance function]\label{def:distance-D-hat-arbitrary-elements}
    Let $G$ be a matrix Lie group. We call $\widehat{D}:G\times G\to \mathbb{R}_+$ an \emph{Element-to-Element} \emph{(EE-)distance function}, a function that measures the distance between elements $\mathbf{V}, \mathbf{W}\in G$ with the following properties: 
    \begin{propertylist}
        \item (Positive Definiteness) $\widehat{D}(\mathbf{V}, \mathbf{W}) \ge 0$ and $\widehat{D}(\mathbf{V}, \mathbf{W})$ $= 0 \iff \mathbf{V}=\mathbf{W}$;\label{prop:Dhat-positive-definite}
        \item (Differentiability) $\widehat{D}$ is at least once differentiable (see \Cref{def:groupdiff}) in both arguments almost everywhere. In addition, there exists $D_{\text{min},\mathcal{C}}>0$ such that the derivative exists for all $\mathbf{V},\mathbf{W}$ such that $0 < \widehat{D}(\mathbf{V},\mathbf{W}) < D_{\text{min},\mathcal{C}}$. Finally, wherever $\Lop_{\mathbf{V}}[\widehat{D}](\mathbf{V},\mathbf{W})$ does not exist at a point $\mathbf{V}=\mathbf{V}_0$, $\mathbf{W}=\mathbf{W}_0$, every directional limit $\mathbf{V} \rightarrow \mathbf{V}_0$, $\mathbf{W} \rightarrow \mathbf{W}_0$ of $\Lop_{\mathbf{V}}[\widehat{D}](\mathbf{V},\mathbf{W})$ exists and is bounded.  The same holds for $\Lop_{\mathbf{W}}[\widehat{D}](\mathbf{V},\mathbf{W})$.
         \label{prop:Dhat-differentiability}
    \end{propertylist}
\hfill$\square$\end{definition}

By allowing the function to be non-differentiable in certain cases, we can incorporate important distance functions. For example, let \(\text{T}(m)\), the \emph{translation group},  be the subgroup of \(\text{SE}(m)\) in which the rotation matrix is the identity. Let \(\mathcal{T}: \text{T}(m) \to \mathbb{R}^m\) be the operator that extracts the translation vector from this matrix. Then, for \(G=\text{T}(m)\), the Euclidean distance $\widehat{D}(\mathbf{V},\mathbf{W}) = \|\mathcal{T}(\mathbf{V}) - \mathcal{T}(\mathbf{W})\|$  is not differentiable at $\mathbf{V}=\mathbf{W}$. However, all directional limits of the derivatives exist and are bounded. Furthermore, although it is not differentiable when $\widehat{D}=0$, it is differentiable everywhere else, so $D_{\text{min},\mathcal{C}}$ can be taken to be $\infty$ . Overall, the (possible) non-differentiability when $\mathbf{V}=\mathbf{W}$  for a generic $\widehat{D}$ will not be an issue, as will be clear soon.

Now, a distance between an element $\mathbf{H}$ to the curve $\mathcal{C}$ is defined as the minimum distance, as measured by $\widehat{D}$, between $\mathbf{H}$ and any $\mathbf{Y}$ in the curve. 
\begin{definition}[EC-distance Function]\label{def:distance-D-element-curve}
    Given an EE-distance function as in \Cref{def:distance-D-hat-arbitrary-elements}, an \emph{Element-to-Curve (EC-)}distance function $D: G\to\mathbb{R}_+$ measures the distance between an element $\mathbf{H}$ and a curve $\mathcal{C}$ parameterized by $\mathbf{H}_d(s)$:
    \begin{align}
        D(\mathbf{H}) \triangleq \min_{\mathbf{Y}\in\mathcal{C}}\widehat{D}(\mathbf{H}, \mathbf{Y}) =
        \min_{s\in[0,1]} \widehat{D}\bigl(\mathbf{H}, \mathbf{H}_d(s)\bigr).\label{eq:optimization-problem-distance-definition-point-curve}
    \end{align}
    Furthermore, let $\mathcal{P}_1\subset G$ be the set of points such that the optimization problem in \eqref{eq:optimization-problem-distance-definition-point-curve} does not have a unique solution. We define $s^*: G \to [0,1]$ so, when $\mathbf{H}\notin \mathcal{P}_1$, $s^*(\mathbf{H})$ is the unique minimizer of the objective function with respect to $s$ for the given $\mathbf{H}$. We also define $\mathcal{P}_2 \subset G \setminus (\mathcal{P}_1 \cup \mathcal{C})$ as the set of elements $\mathbf{H}$ for which the pair $(\mathbf{V},\mathbf{W}) = \bigl(\mathbf{H}, \mathbf{H}_d(s^*(\mathbf{H}))\bigr)$ is a non-differentiability point of $\widehat{D}\bigl(\mathbf{V},\mathbf{W}\bigr)$. Finally, we set $\mathcal{P} \triangleq \mathcal{P}_1 \cup \mathcal{P}_2$.
\hfill$\square$\end{definition}
Note that, from the definition of $\mathcal{P}$, $\mathcal{C} \cap \mathcal{P} = \emptyset$. Furthermore, the fact that $D_{\text{min},\mathcal{C}}>0$ in \Cref{def:distance-D-hat-arbitrary-elements} means that there is a non-zero separation distance between $\mathcal{C}$ and $\mathcal{P}$. For the sake of notational simplicity, we will often omit the dependency of $s^*$ on $\mathbf{H}$ and write simply $s^*$ instead of $s^*(\mathbf{H})$.

Before we define the vector field components, we will state a lemma that will be extensively used to prove many other lemmas and propositions throughout this paper. For that, it will be useful to define $\boldsymbol{\xi}_d$, the \emph{generalized twist} $\boldsymbol{\xi}$ in \eqref{eq:derivative-H-SL-considered-system} to follow the desired curve at a given $\mathbf{H}=\mathbf{H}_d(s)$.
\begin{definition}[Curve twist] \label{def:XId-twist-Hd-for-tangent}
    Let $\boldsymbol{\xi}_d(s) \triangleq \Xi[\mathbf{H}_d](s)$ (see \Cref{def:Xioperator}). Furthermore, we call the parametrization $\mathbf{H}_d(s)$ \emph{proper} if $\boldsymbol{\xi}_d(s) \neq \mathbf{0}$ for all $s \in [0,1]$.
\hfill$\square$\end{definition}
Having a proper parametrization is a purely geometric feature of the curve $\mathcal{C}$. The non-self-intersection property of the curve $\mathcal{C}$ implies that $\mathbf{H}_d: [0,1] \to \mathcal{C}$ is bijective and thus any non-proper parametrization can be transformed into a proper one through reparametrization (e.g., arc length parametrization). With this definition, we can proceed with the following useful lemma.
\begin{lemma}\label{lemma:optimization-problem-part-hdi-vanishers}
    When $\mathbf{H} \notin \mathcal{C} \cup \mathcal{P}$ , the first order optimality condition of \eqref{eq:optimization-problem-distance-definition-point-curve} implies that 
    \begin{align*}
    \Lop_{\mathbf{W}}[\widehat{D}]\bigl(\mathbf{H},\mathbf{H}_d(s^*)\bigr)\boldsymbol{\xi}_d(s^*) = 0.
    \end{align*}
\end{lemma}
\begin{proof}
    Let $\mathbf{H}\in G\setminus(\mathcal{C} \cup \mathcal{P})$ be an arbitrary element, $\mathbf{H}_d(s)\in G$ the parametrization of a curve, and $D$ an EC-distance defined as in \Cref{def:distance-D-element-curve}. Since $D$ is the minimum value of an EE-distance function $\widehat{D}$ (see \Cref{def:distance-D-hat-arbitrary-elements}), by the first order optimality condition of \eqref{eq:optimization-problem-distance-definition-point-curve}, we know that the optimal parameter $s^*$ renders $\frac{d}{ds}\widehat{D}(\mathbf{H}, \mathbf{H}_d(s))|_{s=s^*}=0$. Taking the derivative of $\widehat{D}$, using \Cref{corol:corol1}, at $s=s^*$ gives
    \begin{align*}
        \frac{d}{ds}\widehat{D}\bigl(\mathbf{H}, \mathbf{H}_d(s)\bigr)\bigr|_{s=s^*}=\Lop_\mathbf{W}[\widehat{D}]\bigl(\mathbf{H}, \mathbf{H}_d(s^*)\bigr)\Xi[\mathbf{H}_d](s^*),
    \end{align*}
    since $\mathbf{H}$ does not depend on $s$. Using \Cref{def:XId-twist-Hd-for-tangent} and using the fact that this derivative is zero, we obtain the desired result. Furthermore, since $\mathbf{H} \not\in \mathcal{C} \cup \mathcal{P}$, the function $\widehat{D}$ is guaranteed to be differentiable.  
\end{proof}
\subsection{Normal component}
Following the same steps as in \Cref{sec:adriano-review}, after defining our EC-distance function, we state our vector field components. To define our \emph{normal} component, we will be inspired by \Cref{feat:adriano-time-derivative-lyapunov-normal-comp} (on page \pageref{feat:adriano-time-derivative-lyapunov-normal-comp}), which implies that we identify the normal component by analyzing the expression of $\dot{D}$. More precisely, the normal component is defined as the negative transpose of the term that multiplies the control input $\boldsymbol{\xi}$ in this expression. In that regard, we begin with the following lemma.
\begin{lemma}\label{lemma:time-derivative-of-distance-function}
    When $\mathbf{H}(t) \not \in \mathcal{C} \cup \mathcal{P}$, the time derivative of an EC-distance $D(\mathbf{H}(t))$, under the dynamics in \eqref{eq:derivative-H-SL-considered-system}, is given by 
    \begin{align}
        \frac{d}{dt}{D} = \Lop[D](\mathbf{H})\boldsymbol{\xi} =  \Lop_{\mathbf{V}}[\widehat{D}]\bigl(\mathbf{H},\mathbf{H}_d(s^*)\bigr)\boldsymbol{\xi}. \label{eq:final-equation-for-normal-component}
    \end{align}
\end{lemma}
\begin{proof}
The first part of the equation comes from the chain rule in \Cref{lemma:chainrule} and \eqref{eq:derivative-H-SL-considered-system}. For the second equality, use the fact that $D(\mathbf{H}) = \widehat{D}\bigl(\mathbf{H},\mathbf{H}_d(s^*(\mathbf{H}))\bigr)$ and differentiate applying the chain rule using \Cref{corol:corol1}:
    \begin{align}
    \dot{D} = \Lop_{\mathbf{V}}[\widehat{D}]\boldsymbol{\xi} + \bigg(\Lop_{\mathbf{W}}[\widehat{D}] \boldsymbol{\xi}_d\bigg) \frac{ds^*}{dt} \label{eq:part-of-proof-use-optimallity-cond}
\end{align}
in which the dependencies of $\Lop_{\mathbf{V}}$ and $\Lop_{\mathbf{W}}$ on $\mathbf{H}$ and $\mathbf{H}_d(s^*)$, as well as $\boldsymbol{\xi}_d$ on $s^*$, were omitted. By \Cref{lemma:optimization-problem-part-hdi-vanishers}, the term within parentheses in \eqref{eq:part-of-proof-use-optimallity-cond} vanishes and we obtain the desired result. Note that \eqref{eq:final-equation-for-normal-component} shows that the derivative exists and is continuous whenever the remaining term is continuous; that is, when $\Lop_{\mathbf{V}}[\widehat{D}](\mathbf{H},\mathbf{H}_d(s^*))$ is differentiable, i.e., when $\mathbf{H} \not \in \mathcal{C} \cup \mathcal{P}$. 
\end{proof}

Now, \eqref{eq:final-equation-for-normal-component} in \Cref{lemma:time-derivative-of-distance-function} allows us to define the normal component. 
\begin{definition} [Normal component] \label{def:normal-vector}
    When $\mathbf{H} \not \in \mathcal{C} \cup \mathcal{P}$, the \emph{normal} component of the vector field $\boldsymbol{\xi}_N: G\to\mathbb{R}^m$ is defined as the negative transpose of the term that multiplies $\boldsymbol{\xi}$ in \eqref{eq:final-equation-for-normal-component}, i.e., $\boldsymbol{\xi}_N(\mathbf{H}) {\triangleq} {-}\Lop_{\mathbf{V}}[\widehat{D}]\Bigl(\mathbf{H}, \mathbf{H}_d(s^*)\Bigr)^{{\top}}{.}$
\hfill$\square$\end{definition}
When $\mathbf{H} \in \mathcal{C} \cup \mathcal{P}$, $\boldsymbol{\xi}_N(\mathbf{H})$ is left undefined. As we will see in \Cref{subs:conv-result}, this lack of definition at $\mathcal{C}$ is not an issue, as it will not be necessary to define it in that context. 

Finally, note that the previous definition allows us to express, as in \Cref{feat:adriano-time-derivative-lyapunov-normal-comp}, that $\dot{D} = -\boldsymbol{\xi}_N(\mathbf{H})^{\top}\boldsymbol{\xi}$, provided that $\mathbf{H} \not \in \mathcal{C} \cup \mathcal{P}$. In \cite{Rezende2022}, it was possible to write $\dot{D} = \nabla D\boldsymbol{\xi}$, where the gradient is taken with respect to the state, as the system state in that case lies in $\mathbb{R}^m$, which lacks nonlinear manifold constraints. However, our case is more general since the state $\mathbf{H}$ lies on an $m$-dimensional manifold embedded in a space with higher dimension $n^2$, in which $n$ is the order of the square matrix $\mathbf{H}$. Thus, by comparing $\dot{D} = \nabla D  \boldsymbol{\xi}$ with $\dot{D} = -\boldsymbol{\xi}_N(\mathbf{H})^{\top}\boldsymbol{\xi}$, we can see that $-\boldsymbol{\xi}_N(\mathbf{H})$ serves as the ``gradient of the distance function'' that respects the Lie group structure.

\subsection{Tangent component}
As in \Cref{sec:adriano-review}, the \emph{tangent} component of our vector field is associated solely with the curve. It can be easily defined using \Cref{def:XId-twist-Hd-for-tangent}.
\begin{definition} [Tangent component]\label{def:tangent-vector}
     For $\mathbf{H} \not \in \mathcal{P}$, the \emph{tangent} component of the vector field, $\boldsymbol{\xi}_T:G\to\mathbb{R}^m$ is defined as $\boldsymbol{\xi}_T(\mathbf{H})\triangleq\boldsymbol{\xi}_d(s^*(\mathbf{H}))$. 
\hfill$\square$\end{definition}

\subsection{Orthogonality of components}
According to \Cref{feat:adriano-orthogonality}, it is necessary that the normal and tangent components of the vector field be orthogonal to each other. This fact is related to the EC-distance function $D$, consequently to the EE-distance function $\widehat{D}$, and is a consequence of the \emph{left-invariance} property. We will first define a \emph{left-invariant} distance function and then provide a proposition for the orthogonality condition.
\begin{definition}[Left-invariant distance]\label{def:distance-left-invariant}
    An EE-distance function $\widehat{D}:G\times G\to\mathbb{R}_+$ is said to be  \emph{left-invariant} if it satisfies $\widehat{D}(\mathbf{A}\mathbf{V}, \mathbf{A}\mathbf{W}) = \widehat{D}(\mathbf{V}, \mathbf{W})$ for all $\mathbf{A}, \mathbf{V}, \mathbf{W}\in G$ .
\hfill$\square$\end{definition}
Given this definition, we can state the following:
\begin{proposition}\label{propos:left-invariant-metric-induces-orthogonal}
    Let $\mathbf{H} \not \in \mathcal{C} \cup \mathcal{P}$. If $\widehat{D}$ is a left-invariant distance function, then the vector field components (\Cref{def:normal-vector,def:tangent-vector}) are orthogonal to each other, i.e., $\boldsymbol{\xi}_N(\mathbf{H})^{\top}\boldsymbol{\xi}_T(\mathbf{H})=0$.
\end{proposition}
\begin{proof}
    Since $\widehat{D}$ is left-invariant, $\widehat{D}(\mathbf{A}\mathbf{B}, \mathbf{A}\mathbf{C})=\widehat{D}(\mathbf{B}, \mathbf{C})\;\forall\;\mathbf{A}, \mathbf{B}, \mathbf{C}\in G$. Since this holds for any value, take $\mathbf{B}=\mathbf{H}$, $\mathbf{C}=\mathbf{H}_d(s^*)$, and $\mathbf{A}=\exp\left(\tau\SL[\boldsymbol{\xi}_T]\right)$. Now, due to the left-invariance,
    \begin{align}
    \begin{split}
        \widehat{D}\Bigl(\exp\bigl(\tau\SL[\boldsymbol{\xi}_T]\bigr)\mathbf{H},\, \exp\bigl(\tau\SL[\boldsymbol{\xi}_T]\bigr)\mathbf{H}_d(s^*)\Bigr)\\
        =\widehat{D}\bigl(\mathbf{H}, \mathbf{H}_d(s^*)\bigr) \;\forall\;\mathbf{H}\in G,\,\tau\in\mathbb{R}. \label{eq:prop-orthogonal-first-eq}
    \end{split}
    \end{align}
    Differentiating both sides of this equation with respect to $\tau$, and applying \Cref{corol:corol1} with $\mathbf{A}(\tau) =\exp\left(\tau\SL[\boldsymbol{\xi}_T]\right)$ yields:
    \begin{align*}
        \Lop_{\mathbf{V}}[\widehat{D}]\big(\mathbf{A}(\tau)\mathbf{H},\mathbf{A}(\tau)\mathbf{H}_d(s^*)\big)\Xi[\mathbf{A}(\tau)\mathbf{H}](\tau) {+} \\
        \Lop_{\mathbf{W}}[\widehat{D}]\bigl(\mathbf{A}(\tau)\mathbf{H},\mathbf{A}(\tau)\mathbf{H}_d(s^*)\bigr)\Xi[\mathbf{A}(\tau)\mathbf{H}_d](\tau)= 0, 
    \end{align*}
    since the right-hand side of \eqref{eq:prop-orthogonal-first-eq} does not depend on $\tau$. Note that here the operator $\Xi$ acts with respect to the parameter $\tau$ (i.e., $\sigma = \tau$, see \Cref{def:Xioperator}), since differentiation is carried out with respect to this variable. Nevertheless, computing both $\Xi$ operators (see \Cref{def:Xioperator}) and evaluating both sides at $\tau=0$ results in:
    \begin{align*}
        \Lop_{\mathbf{V}}[\widehat{D}](\mathbf{H},\mathbf{H}_d(s^*))\boldsymbol{\xi}_T {+}\Lop_{\mathbf{W}}[\widehat{D}]\bigl(\mathbf{H},\mathbf{H}_d(s^*)\bigr)\boldsymbol{\xi}_T= 0. 
    \end{align*}
    Noting that, by \Cref{def:tangent-vector}, $\boldsymbol{\xi}_T = \boldsymbol{\xi}_d(s^*)$, and invoking \Cref{lemma:optimization-problem-part-hdi-vanishers}, we have $\Lop_{\mathbf{W}}[\widehat{D}]\boldsymbol{\xi}_T = 0$.
    Finally, using \Cref{def:normal-vector}, we prove the orthogonality property: $\Lop_{\mathbf{V}}[\widehat{D}]\boldsymbol{\xi}_T = -\boldsymbol{\xi}_N^{\top}\boldsymbol{\xi}_T=0$.
\end{proof}

\subsection{Local minima and gradients in distance function}
\Cref{feat:adriano-no-local-minima} is the only feature remaining to be present in our formulation. It consists of two parts: the absence of local minima outside the curve, and the fact that the gradient, here represented by its general form $-\boldsymbol{\xi}_N(\mathbf{H})$, never vanishes (whenever it exists).

To ensure that the EC-distance function has no local minima outside the target curve, we introduce the notion of a \emph{chainable} distance function and show that \emph{chainability} is sufficient to guarantee this property. This property requires defining a \emph{path} between elements in a Lie group.
\begin{definition}[Path generating function]\label{def:PHI-path-parameterizer}
    In a Lie group $G$, a \emph{path generating function} $\Phi:[0, 1] \times G \times G \to G$ is a function that generates a path connecting an element $\mathbf{V}$ to an element $\mathbf{W}$ that satisfies the following properties for all $\mathbf{V}, \mathbf{W}\in G,$ and $\sigma\in[0,1]$:
    \begin{propertylist}
        \item $\Phi(\sigma, \mathbf{V}, \mathbf{W})$ is differentiable in $\sigma$;\label{prop:path-continuous}
        \item $\Phi(0, \mathbf{V}, \mathbf{W}) = \mathbf{V},\,\Phi(1, \mathbf{V},\, \mathbf{W}) = \mathbf{W}$.\label{prop:path-initUfinalV}
    \end{propertylist}
\hfill$\square$\end{definition}

Then:
\begin{definition}[Chainable distance]\label{def:chainable-distance}
    A function $\widehat{D}: G \times G \to \mathbb{R}_+$ is called a \emph{chainable} distance if is an EE-distance function and satisfies the following property: there exists a path generating function $\Phi$ (\Cref{def:PHI-path-parameterizer}), such that for any points $\mathbf{V}, \mathbf{W} \in G$ and any $\sigma \in [0,1]$:
    \begin{align*}
        \widehat{D}(\mathbf{V}{,} \mathbf{W}) {=} \widehat{D}\bigl(\mathbf{V}{,} \Phi(\sigma{,} \mathbf{V}{,} \mathbf{W})\bigr) {+} \widehat{D}\bigl(\Phi(\sigma{,} \mathbf{V}{,} \mathbf{W}){,} \mathbf{W} \bigr).
    \end{align*}%
\hfill$\square$\end{definition}

A chainable distance between two elements can be thought of as a chain, such that it can be broken into pieces within the path generated by the function $\Phi$ and render the same result.

It will now be proved that the chainability property in \Cref{def:chainable-distance} implies the absence of local minima outside the curve in the EC-distance function. This fact implies that all minima of $D$ lie on the curve, and it will become useful when we prove the convergence to the curve.
\begin{figure}[t]
    \centering
    \includeinkscape[pretex=\tiny,width=.8\columnwidth]{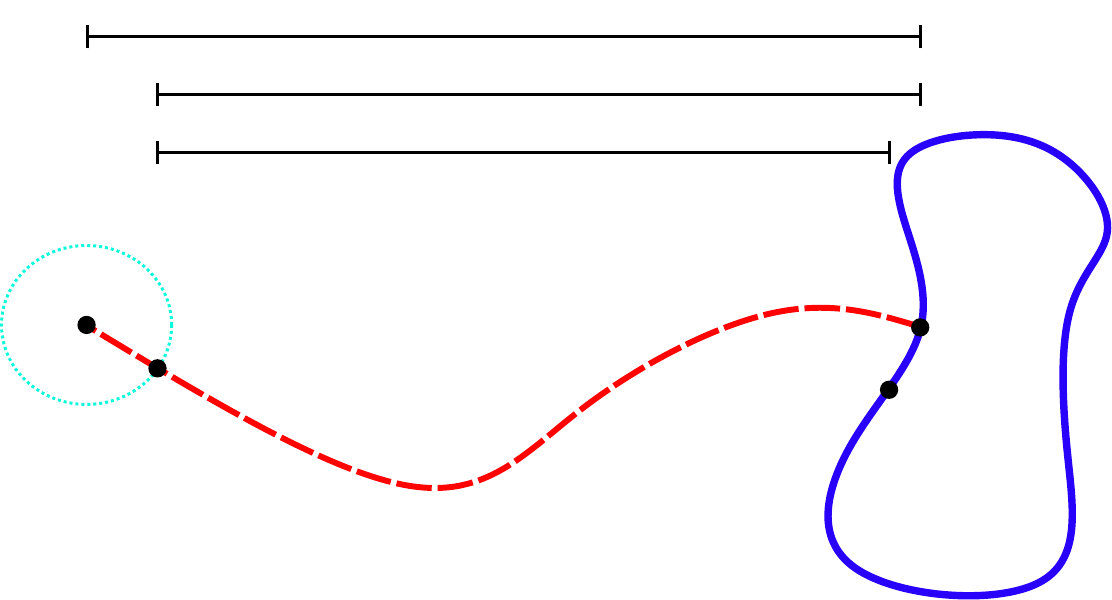}
    \caption{Depiction of the proof in \Cref{propos:D-NO-local-minima}.}
    \label{fig:distance-without-local-minima}
\end{figure}
\begin{proposition}\label{propos:D-NO-local-minima}
    If $\widehat{D}$ is a chainable distance function (\Cref{def:chainable-distance}), then the EC-distance $D$ only has minima on the curve $\mathcal{C}$. Furthermore, all these minima are global.
\end{proposition}
\begin{proof}
    The proof proceeds by contradiction and is depicted in \Cref{fig:distance-without-local-minima}. Take $D$ as an EC-distance function (\Cref{def:distance-D-element-curve}) with $\widehat{D}$ being a chainable EE-distance function. Assume that $\mathbf{B}\notin\mathcal{C}$ is a minimum of $D$ outside the curve, and let $\mathbf{C} \in G$ be (one of) the nearest point(s) on $\mathcal{C}$ to $\mathbf{B}$, i.e., $\mathbf{C}\in\arg\min_{\mathbf{Y}\in\mathcal{C}}\widehat{D}(\mathbf{B}, \mathbf{Y})$.

     Since $\mathbf{B}$ $\not \in \mathcal{C}$ is a minimum of $D$, there exists a set $\mathcal{B}$ in the manifold, small enough to not intersect with the curve but large enough to contain points around \(\mathbf{B}\)\footnote{For example, consider the intersection between a Euclidean ball—centered at \(\mathbf{B}\), with a sufficiently small radius, and defined in the space where the manifold of the Lie group \(G\) resides—and the manifold of \(G\) itself.}, such that $D(\mathbf{Y}) \ge D(\mathbf{B})\; \forall\; \mathbf{Y} \in \mathcal{B}$. Given that the path generated by $\Phi$ between $\mathbf{B}$ and $\mathbf{C}$ is continuous (see \Cref{def:PHI-path-parameterizer}), there must exist $\sigma_0 \in (0,1)$ such that a point $\mathbf{A}=\Phi(\sigma_0, \mathbf{B}, \mathbf{C})$ lies in the intersection of the boundary of the set $\mathcal{B}$ and the path. According to our assumption, $D(\mathbf{B}) \le D(\mathbf{A})$, where $D(\mathbf{A}) = \min_{\mathbf{Y}\in\mathcal{C}} \widehat{D}(\mathbf{A}, \mathbf{Y})$. Let $\mathbf{A}^*$ be the nearest point on the curve to $\mathbf{A}$ such that $D(\mathbf{A})=\widehat{D}(\mathbf{A}, \mathbf{A}^*)$. Now, since $D(\mathbf{A})$ is the minimum EE-distance between $\mathbf{A}$ and the curve, it must be true that this distance is less than or equal to the EE-distance between $\mathbf{A}$ and $\mathbf{C}$, i.e., $D(\mathbf{A}) =\widehat{D}(\mathbf{A}, \mathbf{A}^*)\le \widehat{D}(\mathbf{A}, \mathbf{C})$. The obtained results allow us to obtain the following:
     \begin{align}
         D(\mathbf{B}) \le D(\mathbf{A}) \le \widehat{D}(\mathbf{A}, \mathbf{C}). \label{eq:lemma-local-minima-contradiction-eq1}
     \end{align}
    In addition, by the chainability property (see \Cref{def:chainable-distance}), we also have $D(\mathbf{B}) \ge \widehat{D}(\Phi(\sigma, \mathbf{B}, \mathbf{C}), \mathbf{C})\;\forall\;\sigma\in[0,1]$. Since this holds for any value of $\sigma$, take $\sigma=\sigma_0$, which results in 
    \begin{align}
        D(\mathbf{B}) \ge \widehat{D}(\mathbf{A}, \mathbf{C}). \label{eq:lemma-local-minima-contradiction-eq2}
    \end{align}
    We now have a contradiction. Since conditions \eqref{eq:lemma-local-minima-contradiction-eq1} and \eqref{eq:lemma-local-minima-contradiction-eq2} must hold simultaneously, it follows that $D(\mathbf{B}) = \widehat{D}(\mathbf{A}, \mathbf{C})$. But, from the chainability property (see \Cref{def:chainable-distance}), we know that $D(\mathbf{B}) = \widehat{D}(\mathbf{B}, \mathbf{C}) = \widehat{D}(\mathbf{B}, \mathbf{A}) + \widehat{D}(\mathbf{A}, \mathbf{C})$. This implies that $\widehat{D}(\mathbf{B}, \mathbf{A})=0$. Therefore, since $\widehat{D}$ is positive definite, it follows that $\mathbf{B}=\mathbf{A}$, contradicting the existence of such a set $\mathcal{B}$. Thus, there are no minima outside $\mathcal{C}$.
    
    Now, since $D$ is a non-negative function (see \Cref{def:distance-D-element-curve}), and all the points $\mathbf{H}$ on the curve $\mathcal{C}$ have $D(\mathbf{H})=0$, this implies that all of them are global minima.
\end{proof}

We now establish that, under certain mild conditions, $-\boldsymbol{\xi}_N(\mathbf{H})$, which plays the role of the gradient of the distance in our case, never vanishes wherever it is defined. The necessary ``mild'' condition for establishing this is as follows.

\begin{definition} [Locally linear] \label{def:locallylinear}
    A chainable EE-distance $\widehat{D}$ is said to be \emph{locally linear} if, for any $\mathbf{A}, \mathbf{B} \in G$, $\mathbf{A} \not = \mathbf{B}$:
    \begin{align*}
        \lim_{\sigma \rightarrow 0^+} \frac{1}{\sigma} \widehat{D}\bigl(\mathbf{A},\Phi(\sigma,\mathbf{A},\mathbf{B})\bigr) > 0
    \end{align*}
    \noindent and is finite.
\hfill$\square$\end{definition}
The name arises from the fact that, for any $\mathbf{A} \not= \mathbf{B}$, $\widehat{D}\bigl(\mathbf{A},\Phi(\sigma,\mathbf{A},\mathbf{B})\bigr) \approx K\sigma$ for a constant $K>0$ and $\sigma \approx 0$ (i.e., it is approximately linear near $\sigma=0$).

\begin{lemma} \label{lemma:no-zero-xiN} If $\widehat{D}$ is chainable (\Cref{def:chainable-distance}) and locally linear (\Cref{def:locallylinear}), for any $\mathbf{H} \notin \mathcal{C} \cup \mathcal{P}$, $\boldsymbol{\xi}_N(\mathbf{H}) \not= \mathbf{0}$.
\end{lemma}

\begin{proof}
    Let $\mathbf{H} \notin \mathcal{C} \cup \mathcal{P}$. Choosing $\mathbf{V} = \mathbf{H}$, $\mathbf{W} = \mathbf{H}_d(s^*) = \mathbf{H}^*$, and using the property of chainable functions:
    \begin{align*}        
        \begin{split}
             &D(\mathbf{H}) = \widehat{D}(\mathbf{H},\mathbf{H}^*) =  \\
             &\widehat{D}\bigl(\mathbf{H},\Phi(\sigma,\mathbf{H},\mathbf{H}^*)\bigr){+} \widehat{D}\bigl(\Phi(\sigma,\mathbf{H},\mathbf{H}^*),\mathbf{H}^*\bigr).        
        \end{split}
    \end{align*}
Now, $\widehat{D}\bigl(\Phi(\sigma,\mathbf{H},\mathbf{H}^*),\mathbf{H}^*\bigr) \geq D\bigl(\Phi(\sigma,\mathbf{H},\mathbf{H}^*)\bigr)$. Thus:
\begin{align*}
    D(\mathbf{H})-D\bigl(\Phi(\sigma,\mathbf{H},\mathbf{H}^*)\bigr) \geq \widehat{D}\bigl(\mathbf{H},\Phi(\sigma,\mathbf{H},\mathbf{H}^*)\bigr).
\end{align*}
Divide both sides by $\sigma >0$ and take the limit as $\sigma \rightarrow 0^+$. Since $\mathbf{H} \not \in \mathcal{C}$, $\mathbf{H} \not = \mathbf{H}^*$, and, using \Cref{def:locallylinear}:
\begin{align}
\label{eq:impineq}
 \lim_{\sigma \rightarrow 0^+} \frac{D(\mathbf{H})-D\bigl(\Phi(\sigma,\mathbf{H},\mathbf{H}^*)\bigr)}{\sigma} > 0.   
\end{align}
Let $\boldsymbol{\xi}_\Phi(\sigma,\mathbf{H}) \triangleq \Xi[\Phi(\cdot,\mathbf{H},\mathbf{H}^*(\mathbf{H}))](\sigma)$. Then, from \Cref{lemma:very-important-fact}, it is true that $\Phi(\sigma,\mathbf{H},\mathbf{H}^*) \approx \exp(\SL[\boldsymbol{\xi}_\Phi]\sigma)\mathbf{H}$ for $\sigma \approx 0$. Consequently, the left-hand side of \eqref{eq:impineq} can be written as:
\begin{equation}
\label{eq:impineq2}
  -\lim_{\sigma \rightarrow 0^+} \left( \frac{D\bigl(\exp(\SL[\boldsymbol{\xi}_\Phi]\sigma)\mathbf{H}\bigr) - D(\mathbf{H})}{\sigma} \right).
\end{equation}
Using \Cref{def:Loperator}, this last limit, whenever it exists, is exactly $-\Lop[D](\mathbf{H}) \boldsymbol{\xi}_\Phi$. This limit exists when $\mathbf{H} \not \in \mathcal{C} \cup \mathcal{P}$. But, from \eqref{eq:final-equation-for-normal-component} and \Cref{def:normal-vector}, it can be seen that this limit is also $\boldsymbol{\xi}_N(\mathbf{H})^\top \boldsymbol{\xi}_\Phi$. Thus, from  \eqref{eq:impineq}, $\boldsymbol{\xi}_N(\mathbf{H})^\top \boldsymbol{\xi}_\Phi > 0$,  which implies that $\boldsymbol{\xi}_N(\mathbf{H}) \not= \mathbf{0}$. 
\end{proof}
Intuitively, this lemma means that from any point $\mathbf{H} \notin \mathcal{C} \cup \mathcal{P}$ we can always move in the direction of the closest point $\mathbf{H}^*$ along the path generated by $\Phi$ by applying the twist $\boldsymbol{\xi}_\Phi$ in \eqref{eq:derivative-H-SL-considered-system}. This motion decreases $D$ sufficiently to ensure that the derivative does not vanish, which means that $\boldsymbol{\xi}_N(\mathbf{H})$ cannot be zero, since $\dot{D} = -\boldsymbol{\xi}_N(\mathbf{H})^{\top} \boldsymbol{\xi}$ for any arbitrary twist $\boldsymbol{\xi}$ (\Cref{lemma:time-derivative-of-distance-function}).
\subsection{Convergence results}
\label{subs:conv-result}
With the established definitions, lemmas, and propositions, we can now prove the main result of this text. To do so, we first require the following lemma.

\begin{lemma} \label{lemma:xiNvanishing} If $\widehat{D}$ is an EE-distance function and $k_N$ is defined as in \eqref{eq:vector-field-proposition}, then $\lim_{\mathbf{H} \rightarrow \mathcal{C}} k_N(\mathbf{H}) \boldsymbol{\xi}_N(\mathbf{H}) = \mathbf{0}$.
\end{lemma}
\begin{proof}
    From \Cref{def:normal-vector}, when $\mathbf{H} \rightarrow \mathcal{C}$, the quantity $\Lop_{\mathbf{V}}[\widehat{D}](\mathbf{V},\mathbf{W})$ is evaluated when $\mathbf{V} = \mathbf{W}$ (i.e., $\mathbf{H} = \mathbf{H}_d(s^*)$). According to \Cref{def:distance-D-hat-arbitrary-elements}, this derivative does not necessarily exist, but all directional derivatives exist and are bounded. Since $k_N(\mathbf{H}) = 0$ when $\mathbf{H} \rightarrow \mathcal{C}$, this concludes the proof. 
\end{proof}
This lemma shows that the vector field in \eqref{eq:vector-field-proposition} remains well-defined when $\mathbf{H} \in \mathcal{C}$, even though $\boldsymbol{\xi}_N(\mathbf{H})$ is undefined at these points. Now, we present our theorem:
\begin{theorem}\label{thm:convergence-vector-field}
    Let $\widehat{D}$ be a \emph{left-invariant}, \emph{chainable} and \emph{locally linear} EE-distance function, and $\mathbf{H}_d(s)$ a proper parametrization for $\mathcal{C}$. Then, the closed loop autonomous dynamical system in  \eqref{eq:derivative-H-SL-considered-system} with the input given by  \eqref{eq:vector-field-proposition}, is such that 
    \begin{enumerate}[label=(\roman*)]
        \item the system's state converges either to $\mathcal{C}$ or $\mathcal{P}$;
        \item the set $\mathcal{P}$ is ``escapable'': there exists a policy of choosing arbitrarily small $\boldsymbol{\xi}$ every time $\mathbf{H} \in \mathcal{P}$ such that there exists a finite $t'$ for which $\mathbf{H}(t) \not \in \mathcal{P}$ for all $t \geq t'$;
        \item if the system's state converges to  $\mathcal{C}$, $\mathbf{H}$ traverses the target curve, i.e., $k_T\boldsymbol{\xi}_T$ is never zero, in which $k_T$ is defined as in \eqref{eq:vector-field-proposition}.
    \end{enumerate}
\end{theorem}
\begin{proof}
    \textbf{Statement (i):} Consider $D$ as in \Cref{def:distance-D-element-curve} as a Lyapunov function candidate. According to \Cref{lemma:time-derivative-of-distance-function} and \Cref{def:normal-vector}, its derivative is given by $\dot{D} = -\boldsymbol{\xi}_N^{\top}\boldsymbol{\xi} $ for any $\mathbf{H} \notin \mathcal{C} \cup \mathcal{P}$. Substituting $\boldsymbol{\xi}=\Psi=k_N\boldsymbol{\xi}_N+k_T\boldsymbol{\xi}_T$ and using \Cref{propos:left-invariant-metric-induces-orthogonal}, we obtain 
\begin{align}
\label{eq:Ddotnegative}
    \frac{d}{dt}D(\mathbf{H})= -k_N(\mathbf{H})\|\boldsymbol{\xi}_N(\mathbf{H})\|^2 \;\forall\;\mathbf{H} \notin \mathcal{C} \cup \mathcal{P}
\end{align}
Note, however, that \Cref{lemma:xiNvanishing} guarantees that $\dot{D}$ will also vanish when $\mathbf{H} \in \mathcal{C}$. This fact, along with \eqref{eq:Ddotnegative} and \Cref{lemma:no-zero-xiN}, shows that $\dot{D} < 0$ for all $\mathbf{H} \not \in \mathcal{C} \cup \mathcal{P}$, $\dot{D} = 0$ for $\mathbf{H} \in \mathcal{C}$, and $D$ is non-differentiable when $\mathbf{H} \in \mathcal{P}$. This shows that the system either converges to $\mathcal{C}$ or $\mathcal{P}$.
\begin{figure}[t]
    \centering
    \includeinkscape[pretex=\tiny,width=.8\columnwidth]{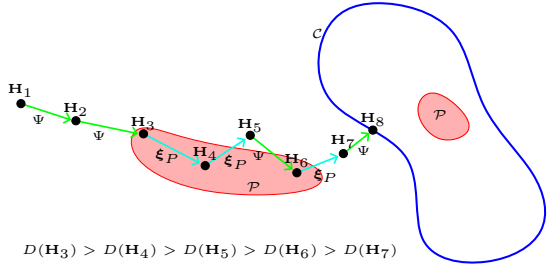}
    \caption{Depiction of Statement (ii) in the proof of \Cref{thm:convergence-vector-field}.}
    \label{fig:proof-escapable}
\end{figure}

\textbf{Statement (ii):} \Cref{propos:D-NO-local-minima} shows that $D$ will not have any local minima outside $\mathcal{C}$. Thus, for each $\mathbf{H} \in \mathcal{P}$, there exists an arbitrarily small perturbation generalized twist $\boldsymbol{\xi}_P(\mathbf{H})$\footnote{To be more precise, \Cref{propos:D-NO-local-minima} shows that one such twist is one that moves $\mathbf{H}$ in the direction of any of the possible $\mathbf{H}_d(s^*)$ through the path induced by $\Phi$.} such that the perturbed state $\mathbf{H}'$ satisfies $D(\mathbf{H}') < D(\mathbf{H})$. Furthermore, there is a non-zero minimum decrease $\delta$ that can be obtained at all steps. Let $t_k$ be the time at which $\mathbf{H}(t)$ enters $\mathcal{P}$ for the $k^{th}$ time, under the application of the controller and the corresponding perturbation policy. Then, we have $D(\mathbf{H}(t_{k+1})) < D(\mathbf{H}(t_k))$, with a decrease of at least $\delta$ at each step. Let $D_{\text{min}, \mathcal{P}} \triangleq \min_{\mathbf{H} \in \mathcal{P}} D(\mathbf{H})$, which is positive since $\mathcal{C} \cap \mathcal{P} = \emptyset$ and $D_{\text{min},\mathcal{C}}$ in \Cref{def:distance-D-hat-arbitrary-elements} is strictly positive. The decreasing sequence $D(t_k)$ must eventually fall below $D_{\text{min}, \mathcal{P}}$ for some finite $k$. From that point onward, since $\dot{D} \leq 0$, $\mathcal{P}$ will not be re-entered. This is illustrated in \Cref{fig:proof-escapable}.

\textbf{Statement (iii):} Traversal comes from the fact that, once in $\mathcal{C}$, the term $k_N \boldsymbol{\xi}_N$ vanishes (\Cref{lemma:xiNvanishing}), while $k_T \boldsymbol{\xi}_T$---the necessary twist to track the curve in a given sense (clockwise or counter-clockwise)---remains non-zero. This non-zero value, due to $k_T$ being positive and that $\mathbf{H}_d$ being a proper parametrization (see \Cref{def:XId-twist-Hd-for-tangent}), enforces the traversal of the curve.
\end{proof} 
Note that the convergence of our field is ``almost global'', since if initialized in the set \(\mathcal{P}\) and left undisturbed, the system will not reach \(\mathcal{C}\). However, as explained in the proof of \Cref{thm:convergence-vector-field}, there always exists a finite sequence of maneuvers (arbitrarily small control inputs $\boldsymbol{\xi}$) that allows the system to escape $\mathcal{P}$ and converge to the curve $\mathcal{C}$.

Note that the direction of traversal (clockwise or counterclockwise) is determined by the choice of parametrization $\mathbf{H}_d(s)$. For instance, using the reparametrization $\mathbf{H}_{d,\text{new}}(s) = \mathbf{H}_d(1-s)$ results in traversal in the opposite direction.

Result (ii) in \Cref{thm:convergence-vector-field} implies that if the system enters the ``problematic set'' $\mathcal{P}$, there always exists an arbitrarily small sequence of maneuvers that enables the system to eventually escape this set in finite time and never return. Furthermore, the following corollary provides the conditions 
under which we have asymptotic stability to the desired curve.
\begin{corollary}\label{cor:local-asymptotic-stability}
Let the assumptions of \Cref{thm:convergence-vector-field} hold and define $D_{\min,\mathcal{P}} = \min_{\mathbf{H}\in\mathcal{P}} D(\mathbf{H})$. If the initial state satisfies $D(\mathbf{H}(0)) < D_{\min,\mathcal{P}}$, then the closed-loop system converges asymptotically to the curve $\mathcal{C}$ and the problematic set $\mathcal{P}$ is never reached. In particular, $\mathcal{C}$ is locally asymptotically stable with a region of attraction that contains the sublevel set $\{\mathbf{H}\in G \mid D(\mathbf{H}) < D_{\min,\mathcal{P}}\}$.
\end{corollary}
\begin{proof}
Because $\dot{D} \le 0$ for all $\mathbf{H} \notin \mathcal{P}$ and $D$ decreases strictly outside $\mathcal{C}\cup\mathcal{P}$, the condition $D(\mathbf{H}(0)) < D_{\min,\mathcal{P}}$ guarantees that the state can never enter $\mathcal{P}$ (where the distance would be at least $D_{\min,\mathcal{P}}$). From \Cref{thm:convergence-vector-field}(i), the state must therefore converge to $\mathcal{C}$.
\end{proof}

\subsection{An explicit construction for the EE-distance}
\label{subs:explconst}
The proposed vector field strategy was developed for any connected matrix Lie group. In this section, we define a path generating function \(\Phi\) and EE-distance function specifically for a class of groups that we call here \emph{special exponential} Lie groups, which are among the most common in engineering applications.

A left-invariant metric $\mu$ on the tangent space of a Lie group induces a \emph{geodesic distance} $\widehat{D}_{\mu}(\mathbf{V},\mathbf{W})$ between points $\mathbf{V},\mathbf{W}\in G$ that possesses all the properties required in \Cref{def:distance-D-hat-arbitrary-elements} \cite[p. 650]{Gallier2020}. In this setting, one can define the path function $\Phi(\sigma,\mathbf{V},\mathbf{W})$ as the point obtained after traveling a fraction $\sigma$ of the geodesic from $\mathbf{V}$ to $\mathbf{W}$ (e.g., $\Phi(0.7,\mathbf{V},\mathbf{W})$ is the point $70\%$ along the geodesic). The geodesic distance itself could then serve as our function $\widehat{D}$. However, computing geodesic distances is often computationally expensive and, in general, lacks a closed-form expression.

For certain groups and appropriate choices of metric, this geodesic distance can be computed analytically. For example, when $\text{SO}(3)$ is endowed with the bi-invariant metric induced by the Frobenius norm, one obtains the simple formula
\begin{align}
    \widehat{D}_\mu(\mathbf{V},\mathbf{W}) = \|\log(\mathbf{V}^{-1}\mathbf{W})\|_F. \label{eq:geodesic-length-distance}
\end{align}

For more general groups $G$, the same expression $\|\log(\mathbf{V}^{-1}\mathbf{W})\|_F$ may no longer correspond to a geodesic distance---as is the case for $\text{SE}(3)$---yet it can still serve as a valid function $\widehat{D}$ under suitable conditions. This section establishes when such an extension is possible.

Let $G^\circ$ denote the subset of $G$ comprising matrices with no real negative eigenvalues, and let $\mathfrak{g}^\circ \subset \mathfrak{g}$ be the set of matrices whose eigenvalues lie in the \emph{open} strip $\{\lambda : -\pi < \text{Im}(\lambda) < \pi\}$. Let $\bar{\mathfrak{g}}^\circ \subset \mathfrak{g}$ be the set of matrices whose eigenvalues lie in the \emph{closed} strip $\{\lambda : -\pi \le \text{Im}(\lambda) \le \pi\}$. We now formalize the class of groups for which the metric in \eqref{eq:geodesic-length-distance} can be adopted.

\begin{definition}[Special Exponential Group]
\label{def:special-exponential-group}
    For $\mathbf{Z} \in G$, let
    \begin{align*}
        \exp^{-1}(\mathbf{Z}) \triangleq \{ \mathbf{Y} \in \mathfrak{g} \ | \ \exp(\mathbf{Y}) = \mathbf{Z} \} .
    \end{align*}
     A Lie group $G$ is called \emph{special exponential} if it satisfies:
     \begin{propertylist}
         \item $G$ is \emph{exponential}\footnote{Here we adopt the surjectivity definition from \cite{djokovic1995exponential}. The term is also used elsewhere to mean that the exponential map is a diffeomorphism \cite[p. 63]{Onishchik1994}.}, i.e., for every $\mathbf{Z}\in G$, 
         $\exp^{-1}(\mathbf{Z}) \neq \emptyset$;
         \item Every $\mathbf{Z} \in G$ has a preimage in the closed strip, i.e., $\exp^{-1}(\mathbf{Z}) \cap \bar{\mathfrak{g}}^\circ \neq \emptyset$;
         \item For any two elements $\mathbf{Y}_1, \mathbf{Y}_2 \in \exp^{-1}(\mathbf{Z}) \cap \bar{\mathfrak{g}}^\circ$, their Frobenius norms coincide: $\|\mathbf{Y}_1\|_F = \|\mathbf{Y}_2\|_F$.
     \end{propertylist}
    \hfill$\square$
\end{definition}
Not all groups are special exponential, but important groups such as $\text{SO}(m)$, $\text{SE}(m)$ and $\text{SGal}(m)$ are.

For a special exponential group $G$, consider the restriction of the exponential map to the open strip: $\exp: \mathfrak{g}^\circ \to G^\circ$. The special exponential property guarantees that this restricted map is \emph{surjective}. Moreover, a classical result in matrix theory establishes that it is also \emph{injective}. Both facts ensure \emph{bijectivity}, and thus there exists the inverse function $\Log: G^\circ \to \mathfrak{g}^\circ$, referred to as the \emph{principal logarithm} \cite[p. 320]{Gallier2020}.

Having established the principal logarithm $\Log$, we now seek to extend the notion of a logarithm to the entire group $G$. For elements outside $G^\circ$, the principal logarithm is not defined, and multiple choices of preimage in $\bar{\mathfrak{g}}^\circ$ may exist. A \emph{special branch} of the logarithm formalizes one such consistent choice.

\begin{definition}[Special Branch of the Logarithm]
\label{def:special-branch-logarithm}
Let $G$ be a special exponential matrix Lie group. A \emph{special branch} of the matrix logarithm is a function $\log: G \to \bar{\mathfrak{g}}^{\circ}$ satisfying:
\begin{propertylist}
    \item $\log(\mathbf{Z}) = \Log(\mathbf{Z})$ for all $\mathbf{Z} \in G^{\circ}$;
    \item Otherwise, $\log(\mathbf{Z})$ is a chosen element $\mathbf{Y} \in \exp^{-1}(\mathbf{Z}) \cap \bar{\mathfrak{g}}^{\circ}$
\end{propertylist}
    \hfill$\square$
\end{definition}
By fixing, according to a chosen rule, a specific $\mathbf{Y}(\mathbf{Z})$ for each $\mathbf{Z} \notin G^{\circ}$, we obtain a well-defined function $\log: G \to \bar{\mathfrak{g}}^{\circ}$. For example, when $G = \text{SO}(2)$, the only element not in $G^{\circ}$ is $-\mathbf{I}_{2 \times 2}$. There are two possible choices of $\mathbf{Y} \in \bar{\mathfrak{g}}^{\circ}$ such that $\exp(\mathbf{Y}) = -\mathbf{I}_{2 \times 2}$:
\begin{align*}
    \begin{bmatrix} 0 & -\pi \\ \pi & 0 \end{bmatrix} \quad \text{or} \quad \begin{bmatrix} 0 & \pi \\ -\pi & 0 \end{bmatrix}.
\end{align*}
Hence there are two possible special branches, $\log_1$ and $\log_2$, depending on which preimage is chosen for $\log(-\mathbf{I}_{2 \times 2})$. However, regardless of the chosen branch, the Frobenius norm of $\log(-\mathbf{I}_{2 \times 2})$ is the same, reflecting the third property in \Cref{def:special-exponential-group}. Indeed, because of this property and the fact that we will only be concerned with the norm of the logarithm, the choice of special branch can be made arbitrarily.

This extension of the logarithm leads to a key result that will be essential for establishing the properties of our distance.
\begin{lemma}\label{lemma:log-exp-log-equals-log}
    Let $G$ be a special exponential Lie Group. For a chosen special branch $\log$, the following holds: for all $\mathbf{X} \in G$ and all $r \in [0,1]$,
    \begin{align*}
        \Big\|\log\Bigl(\exp\bigl(r\log(\mathbf{X})\bigr)\Bigr)\Big\|_F = \big\|r\log(\mathbf{X})\big\|_F.
    \end{align*}
\end{lemma}
\begin{proof}
    Let $\mathbf{Y}_1 = \log\Bigl(\exp\bigl(r\log(\mathbf{X})\bigr)\Bigr)$ and $\mathbf{Y}_2 = r\log(\mathbf{X})$. By construction, $\exp(\mathbf{Y}_1) = \exp(\mathbf{Y}_2)$, since $\exp(\log(\mathbf{W})) = \mathbf{W}$ for any $\mathbf{W}$ in the domain of $\log$ (the logarithm is a right inverse of the exponential).

    Since $\log$ is a special branch, we have $\mathbf{Y}_1\in \bar{\mathfrak{g}}^{\circ}$. Furthermore, $\log(\mathbf{X}) \in \bar{\mathfrak{g}}^{\circ}$ by definition, and because $0 \leq r \leq 1$, it follows that $\mathbf{Y}_2 =  r\log(\mathbf{X}) \in \bar{\mathfrak{g}}^{\circ}$ as well. Thus, both $\mathbf{Y}_1$ and $\mathbf{Y}_2$ lie in $\exp^{-1}(\mathbf{Z})\cap\bar{\mathfrak{g}}^\circ$, in which  $\mathbf{Z} = \exp(\mathbf{Y}_1) = \exp(\mathbf{Y}_2)$. By the third property of a special exponential group (\Cref{def:special-exponential-group}), they must have the same Frobenius norm, which completes the proof.
\end{proof}

For the special exponential Lie groups and a choice of special branch, a path generating function $\Phi_\sigma\triangleq\Phi(\sigma, \mathbf{V}, \mathbf{W})$ can be defined as:
\begin{align}
    \Phi(\sigma, \mathbf{V}, \mathbf{W}) = \mathbf{V}\exp{\left(\log{\left(\mathbf{V}^{-1}\mathbf{W}\right)}\sigma\right)}, \label{eq:PHI-path-parameterizer-utilized-exp-of-log}
\end{align}
which is in accordance with \Cref{def:PHI-path-parameterizer}. Then, an EE-distance function is defined as:
\begin{align}
    \widehat{D}(\mathbf{V}, \mathbf{W}) = \|\log{(\mathbf{V}^{-1}\mathbf{W})}\|_F.\label{eq:distance-D-hat-utilized-log-norm}
\end{align}

Note that, due to the third property of a special exponential group in \Cref{def:special-exponential-group} , the choice of the (special) branch of logarithm in \Cref{def:special-branch-logarithm} is inconsequential for $\widehat{D}$, since regardless of the choice of logarithm for $\mathbf{V}^{-1}\mathbf{W}$, the Frobenius norm is the same. 

In order to invoke \Cref{thm:convergence-vector-field}, function $\widehat{D}$ in \eqref{eq:distance-D-hat-utilized-log-norm} needs to be an EE-distance (see \Cref{def:distance-D-hat-arbitrary-elements}) that is left-invariant (see \Cref{def:distance-left-invariant}), chainable (see \Cref{def:chainable-distance}) and locally linear (see \Cref{def:locallylinear}). Thus, we prove all of these properties in the following proposition.

\begin{proposition}\label{prop:EE-dist-prop-exponential-group-D-hat}
    Adopting the path $\Phi$ in \eqref{eq:PHI-path-parameterizer-utilized-exp-of-log}, the function $\widehat{D}$ in \eqref{eq:distance-D-hat-utilized-log-norm} is a left-invariant, chainable, and locally linear EE-distance.
\end{proposition}
\begin{proof}
    We prove each property separately. 
    
    \textbf{Positive definiteness}: Positive definiteness can be easily observed.
    
    \textbf{Left-invariant}: The distance function is left-invariant since, for all $\mathbf{A} \in G$, $\widehat{D}(\mathbf{A}\mathbf{V}, \mathbf{A}\mathbf{W}) =  \|\log{(\mathbf{V}^{-1}\mathbf{A}^{-1}\mathbf{A}\mathbf{W})}\|_F$, which is clearly equal to $\widehat{D}(\mathbf{V}, \mathbf{W})$.

    \textbf{Chainable}: To prove the chainability property, we first substitute $\Phi$ by its expression \eqref{eq:PHI-path-parameterizer-utilized-exp-of-log} in $\widehat{D}(\mathbf{V}, \Phi_\sigma)$, which results in
    \begin{align*}
        \begin{split}
             \widehat{D}(\mathbf{V}, \Phi_\sigma) &= \|\log\bigl(\mathbf{V}^{-1}\mathbf{V}\exp(\log(\mathbf{V}^{-1}\mathbf{W})\sigma)\bigr)\|_F\\
             &=\sigma\|\log{\left(\mathbf{V}^{-1}\mathbf{W}\right)}\|_F,
        \end{split}
    \end{align*}
    using \Cref{lemma:log-exp-log-equals-log} with $\mathbf{Z}=\mathbf{V}^{-1}\mathbf{W}$ and $r=\sigma$, and the fact that $\sigma\ge0$. Now, using the fact that, by definition, $\mathbf{V}^{-1}\mathbf{W}=\exp(\log(\mathbf{V}^{-1}\mathbf{W}))$, we can express the following:
    \begin{align}
             \widehat{D}(\Phi_\sigma, \mathbf{W}) = \Bigl\|\log\Bigl(\exp\bigl(-\log(\mathbf{V}^{-1}\mathbf{W})\sigma\bigr)\!\!\!\!\!\! \underbrace{\mathbf{V}^{-1}\mathbf{W}}_{\exp(\log(\mathbf{V}^{-1}\mathbf{W}))}\!\!\!\!\!\!\Bigr)\Bigr\|_F.\nonumber
    \end{align}
    Note that $\log(\mathbf{V}^{-1}\mathbf{W})$ commutes with $-\log(\mathbf{V}^{-1}\mathbf{W})\sigma$, and thus we can express the product of exponentials as the exponential of the sum of the arguments:
    \begin{align}
        \widehat{D}(\Phi_\sigma, \mathbf{W}) {=} \Bigl\|\log\Bigl(\exp\bigl((1-\sigma)\log(\mathbf{V}^{-1}\mathbf{W})\bigr)\Bigr)\Bigr\|_F. \label{eq:proof-chainable-eedist-exp}
    \end{align}
    Invoking \Cref{lemma:log-exp-log-equals-log} with $\mathbf{Z}=\mathbf{V}^{-1}\mathbf{W}$ and $r=1-\sigma$, and using the fact that $0\le\sigma\le1$, expression \eqref{eq:proof-chainable-eedist-exp} reduces to 
    \begin{align*}
       \widehat{D}(\Phi_\sigma, \mathbf{W}) =(1-\sigma)\|\log{\left(\mathbf{V}^{-1}\mathbf{W}\right)}\|_F.
    \end{align*}
    Clearly, $\widehat{D}(\mathbf{V}, \Phi_\sigma) + \widehat{D}(\Phi_\sigma, \mathbf{W}) = \|\log(\mathbf{V}^{-1}\mathbf{W})\|_F = \widehat{D}(\mathbf{V}, \mathbf{W})$. 
    
    \textbf{Locally linear}: to prove that $\widehat{D}$ is locally linear, first note that, using \Cref{lemma:log-exp-log-equals-log} and the fact that $\sigma$ is non-negative, $\widehat{D}(\mathbf{V}, \Phi_\sigma) = \sigma\|\log{\left(\mathbf{V}^{-1}\mathbf{W}\right)}\|_F$, thus we have
    \begin{align*}
        \begin{split}
            \lim_{\sigma\to0^+}\frac{1}{\sigma}\widehat{D}(\mathbf{V}, \Phi_\sigma) &= \lim_{\sigma\to0^+}\frac{\sigma}{\sigma}\|\log{\left(\mathbf{V}^{-1}\mathbf{W}\right)}\|_F\\
            &
            = \|\log{\left(\mathbf{V}^{-1}\mathbf{W}\right)}\|_F > 0
        \end{split}
    \end{align*}
    as long as $\mathbf{V} \not= \mathbf{W}$.
\end{proof}
With the results in \Cref{prop:EE-dist-prop-exponential-group-D-hat}, we can invoke \Cref{thm:convergence-vector-field} and guarantee convergence and traversal along a curve in special exponential Lie groups. These results show that the EE-distance \eqref{eq:distance-D-hat-utilized-log-norm} is a general EE-distance function for our approach in scenarios involving special exponential Lie groups. Furthermore, this EE-distance can be specialized to particular Lie groups, as shown in \Cref{subs:se3-dist-EE} for the case of $\text{SE}(3)$.

\subsection{Particular case of Euclidean space}
\label{subs:partcase}

In this section, we present our results to the application in Euclidean space, which reduces our strategy to the one in \cite{Rezende2022}. Since our work relies on matrix Lie groups, we must first represent $\mathbb{R}^m$ by the Lie group $\text{T}(m)$, the \emph{$m$-dimensional translation group}. This group is a matrix Lie subgroup of $\text{SE}(m)$ in which the rotation part of the $n$-dimensional homogeneous transformation matrix (with $n=m+1$) is the identity matrix. To relate matrices in $\text{T}(m)$ with vectors in $\mathbb{R}^m$, we use the isomorphism $\mathcal{T}: \text{T}(m) \to \mathbb{R}^m$, where $\mathcal{T}(\mathbf{H})$ is obtained by extracting the first $m$ elements of the last column of $\mathbf{H}\in\text{T}(m)$. Matrices in this group have the following structure:
\begin{align*}
    \mathbf{H} = \begin{bmatrix}
        \mathbf{I} & \mathbf{p}\\ \mathbf{0} & 1
    \end{bmatrix} \in \text{T}(m),\ \mathbf{p}\in\mathbb{R}^m.
\end{align*}

Henceforth, we take as a basis of the Lie algebra of $\text{T}(m)$, $\mathfrak{t}(m)$, the matrices $\mathbf{E}_k$, $k \in \{1,2,\dots,m\}$, where $\mathbf{E}_k$ has all entries $0$ except for the $k^{th}$ entry of the last column, which is $1$. With this choice, the system \eqref{eq:derivative-H-SL-considered-system} reduces to the single integrator model used in \cite{Rezende2022}, as $\frac{d}{dt} \mathcal{T}(\mathbf{H}) = \boldsymbol{\xi}$.

We begin by showing that using the path generating function defined in \eqref{eq:PHI-path-parameterizer-utilized-exp-of-log} and the EE-distance defined in \eqref{eq:distance-D-hat-utilized-log-norm} renders the Euclidean distance between vectors as shown in \Cref{sec:adriano-review}. First, note that the exponential map in this case is given by $\exp(\mathbf{A})=\mathbf{A}+\mathbf{I},\, \forall\,\mathbf{A}\in\mathfrak{t}(m)$, and the logarithm is given by $\log(\mathbf{X})=\mathbf{X}-\mathbf{I},\, \forall\,\mathbf{X}\in\text{T}(m)$.

Let $\mathbf{V},\mathbf{W}\in\text{T}(m)$ such that $\mathcal{T}(\mathbf{V}) = \mathbf{v}\in\mathbb{R}^m$, $\mathcal{T}(\mathbf{W}) = \mathbf{w}\in\mathbb{R}^m$. The path generating function in \eqref{eq:PHI-path-parameterizer-utilized-exp-of-log} is given by
\begin{align}
    \begin{split}
        \mathbf{V}\exp{\bigl(\log{(\mathbf{V}^{-1}\mathbf{W})}\sigma\bigr)} 
        = \begin{bmatrix}
            \mathbf{I} & (1 - \sigma)\mathbf{v} + \sigma \mathbf{w}\\ \mathbf{0} & 1
        \end{bmatrix}. \label{eq:euclidean-adriano-path-matrix}
        \end{split}
\end{align}
Applying \eqref{eq:distance-D-hat-utilized-log-norm}, $\|\log{(\mathbf{V}^{-1}\mathbf{W})}\|_F$$= \|\mathbf{V}^{-1}\mathbf{W} - \mathbf{I}\|_F$. Note that $\mathbf{V}^{-1}\mathbf{W} - \mathbf{I}$ is a matrix whose only non-zero column is the last one, equal to $[\,(\mathbf{w} - \mathbf{v})^\top\quad 0\,]^\top$. This implies that $\|\mathbf{V}^{-1}\mathbf{W} - \mathbf{I}\|_F$$=\|\mathbf{w}-\mathbf{v}\|$, which shows that the EE-distance in \eqref{eq:distance-D-hat-utilized-log-norm} reduces to the Euclidean distance, as expected, i.e., $\widehat{D}(\mathbf{V},\mathbf{W})=\|\mathcal{T}(\mathbf{V}) - \mathcal{T}(\mathbf{W})\|$. Furthermore, the EC-distance is the same as defined in \Cref{sec:adriano-review}.

The normal component is given by the gradient of the EC-distance, thus $\boldsymbol{\xi}_{N}(\mathbf{H})= \Bigl(\mathcal{T}\bigl(\mathbf{H}_d(s^*)\bigr) - \mathcal{T}(\mathbf{H})\Bigr)/\|\mathcal{T}\bigl(\mathbf{H}_d(s^*)\bigr) - \mathcal{T}(\mathbf{H})\|$, i.e., the normalized vector that points from the current point to the nearest point on the curve.

In this scenario, the tangent component is precisely the tangent vector of the curve at the nearest point, so $\boldsymbol{\xi}_{T}=\left.\frac{d}{ds}\mathcal{T}(\mathbf{H}_d(s))\right|_{s=s^*}$.

The normal and tangent components are orthogonal, since the EE-distance is left-invariant, which we now show. Note that $\mathcal{T}(\mathbf{X}\mathbf{V}) = \mathcal{T}(\mathbf{X}) + \mathcal{T}(\mathbf{V})\;\forall\;\mathbf{V},\mathbf{X}\in\text{T}(m)$. Let $\mathbf{V}, \mathbf{W},\mathbf{X} \in \text{T}(m)$, thus, $\widehat{D}(\mathbf{X}\mathbf{V}, \mathbf{X}\mathbf{W}) = \|\mathcal{T}(\mathbf{X}\mathbf{V}) - \mathcal{T}(\mathbf{X}\mathbf{W})\| = \|\mathcal{T}(\mathbf{X}) + \mathcal{T}(\mathbf{V}) - \bigl(\mathcal{T}(\mathbf{X}) + \mathcal{T}(\mathbf{W})\bigr)\| = \widehat{D}(\mathbf{V}, \mathbf{W})$.

We now show that the EC-distance has no local minima outside the curve by showing the chainability property. Let $\Phi_{\sigma} = \Phi(\sigma,\mathbf{V},\mathbf{W})$. From \eqref{eq:euclidean-adriano-path-matrix}, we know that $\mathcal{T}(\Phi_{\sigma})$ = $(1 - \sigma)\mathcal{T}(\mathbf{V}) + \sigma \mathcal{T}(\mathbf{W})\;\forall\;\mathbf{V},\mathbf{W}\in \text{T}(m)$. Using the fact that $0\le \sigma \le 1$, we have
\begin{align}
        \begin{split}
            \widehat{D}(\mathbf{V}, \Phi_\sigma) &= \|\mathcal{T}(\mathbf{V}) - (1 - \sigma)\mathcal{T}(\mathbf{V}) - \sigma \mathcal{T}(\mathbf{W})\|\\
            &
            =\sigma\|\mathcal{T}(\mathbf{V})-\mathcal{T}(\mathbf{W})\|,  \label{eq:example-adriano-chainable-DvPhi}
        \end{split}
        \\
        \begin{split}
            \widehat{D}(\Phi_\sigma, \mathbf{W}) &= \|(1 - \sigma)\mathcal{T}(\mathbf{V}) + \sigma \mathcal{T}(\mathbf{W}) - \mathcal{T}(\mathbf{W})\|\\
            &
            =(1-\sigma)\|\mathcal{T}(\mathbf{V})-\mathcal{T}(\mathbf{W})\|.
        \end{split}
    \end{align}
    Summing both terms, results in $\widehat{D}(\mathbf{V}, \Phi_\sigma)+\widehat{D}(\Phi_\sigma, \mathbf{W})=\|\mathcal{T}(\mathbf{V})-\mathcal{T}(\mathbf{W})\|=\widehat{D}(\mathbf{V},\mathbf{W})$.

The EE-distance $\widehat{D}(\mathbf{V},\mathbf{W})=\|\mathcal{T}(\mathbf{V}) - \mathcal{T}(\mathbf{W})\|$ is locally linear. From \eqref{eq:example-adriano-chainable-DvPhi}, we know that $\widehat{D}(\mathbf{V}, \Phi_\sigma) = \sigma\|\mathcal{T}(\mathbf{V}) - \mathcal{T}(\mathbf{W})\|$, thus $\lim_{\sigma \to 0^+} \frac{1}{\sigma} \widehat{D}\bigl(\mathbf{V},\Phi_\sigma\bigr) = \lim_{\sigma \to 0^+} \frac{1}{\sigma}\sigma\|\mathcal{T}(\mathbf{V}) - \mathcal{T}(\mathbf{W})\| = \|\mathcal{T}(\mathbf{V}) - \mathcal{T}(\mathbf{W})\| > 0$ for $\mathbf{V} \not= \mathbf{W}$, and thus $\widehat{D}$ is locally linear.

These results allow us to invoke \Cref{thm:convergence-vector-field} and reduce our results to the ones in \cite{Rezende2022}. Furthermore, this implies that we can use the vector field to guide positions or $m$-dimensional configurations to converge to and traverse along a predefined curve.

An interesting aspect of the Euclidean space setting is the choice of distance functions. While \cite{Rezende2022} fixes the EE-distance as the Euclidean distance, the properties we require for the EE-distance function in this paper show that it does not necessarily have to be Euclidean. As long as the distance function satisfies left-invariance, chainability, and local linearity, it is suitable for our framework. For instance, for any norm \(\|\cdot\|_P\) on \(\mathbb{R}^n\) that is differentiable except at the origin, we can take \(\widehat{D}(\mathbf{V}, \mathbf{W}) = \|\mathcal{T}(\mathbf{V}) - \mathcal{T}(\mathbf{W})\|_P\).

\section{Experimental Validation}\label{sec:simulation}
To validate our control strategy, we conducted a real-world experiment using a Kinova\textsuperscript{\tiny\textregistered}~Gen3 robotic manipulator with 7 degrees of freedom. The goal was for the manipulator's end-effector to track a target curve in $\text{SE}(3)$. We used Kinova\textsuperscript{\tiny\textregistered}~Kortex\textsuperscript{\texttrademark}~API\footnote{\url{https://github.com/Kinovarobotics/Kinova-kortex2_Gen3_G3L}} to control the robot via joint velocities. Since, our control strategy provides linear and angular velocities, the following steps were necessary to perform the experiment:
\begin{itemize}
    \item Design a curve in $\text{SE}(3)$ that avoids both self-collisions and joint limit violations;
    \item Compute the twist $\boldsymbol{\xi}=[\mathbf{v}^\top\quad \boldsymbol{\omega}^\top]^\top$ generated by the vector field strategy;
    \item Compute the end-effector's linear velocity $\dot{\mathbf{t}}=\boldsymbol{\omega} \times \mathbf{t} + \mathbf{v}$ and let $\boldsymbol{\xi}'=[\dot{\mathbf{t}}^\top\quad \boldsymbol{\omega}^\top]^\top$, where $\mathbf{t}$ is the position of the end-effector;
    \item Compute the joint velocities $\dot{\mathbf{q}} = \bigl(\mathbf{J}(\mathbf{q})^\top\mathbf{J}(\mathbf{q})+\varepsilon\mathbf{I}\bigr)^{-1}\mathbf{J}(\mathbf{q})^\top\boldsymbol{\xi}'$, where $\mathbf{J}(\mathbf{q})$ is the geometric Jacobian of the end-effector, and $\varepsilon=\num{1E-4}$;
    \item Send the joint velocities to the manipulator and receive the current configuration data.
\end{itemize}
More details on these steps are provided in the subsequent sections.

\subsection{The EE-distance function}\label{subs:se3-dist-EE}
As mentioned, the group \text{SE}(3) is exponential, allowing us to utilize the construction from \Cref{subs:explconst}. However, instead of computing $\widehat{D}(\mathbf{V},\mathbf{W}) = \|\log(\mathbf{V}^{-1}\mathbf{W})\|_F$ through a generic algorithm to compute the matrix logarithm followed by applying the matrix norm, the structure of the group $\text{SE}(3)$ allows a more efficient approach. The algorithm for computing $\widehat{D}(\mathbf{V},\mathbf{W})$ is as follows:

\begin{itemize}
    \item Extract the rotation part $\mathbf{Q} \in \text{SO}(3)$ and the translation part $\mathbf{t} \in \mathbb{R}^3$ out of $\mathbf{V}^{-1}\mathbf{W}$.

    \item Compute\footnote{This comes from Rodrigues' rotation formula.} $u \triangleq \frac{1}{2}\bigl(\tr(\mathbf{Q})-1\bigr)$ and $v \triangleq \frac{1}{2\sqrt{2}}\|\mathbf{Q}{-}\mathbf{Q}^{\top}\|_F$. One can see that $u = \cos(\theta)$ and $v=\sin(\theta)$, in which $\theta \in [0, \pi]$ is the rotation angle related to $\mathbf{Q}$. Compute $\theta = \text{atan2}(v,u)$.

    \item Compute $\alpha \triangleq \frac{2-2u-\theta^2}{4(1-u)^2}$. Compute
    $\mathbf{M} \triangleq  \mathbf{I}(1-2\alpha) +  (\mathbf{Q}+\mathbf{Q}^{\top})\alpha$.

    \item Finally, $\widehat{D}= \sqrt{2\theta^2 + \mathbf{t}^\top \mathbf{M}\mathbf{t}}$.

\end{itemize}

As mentioned before, the third property of a special exponential group (\Cref{def:special-exponential-group}) ensures that $\widehat{D}$ does not depend on the particular choice of special branch of the logarithm. This is reflected in the algorithm above by the fact that no such choice appears in its steps.

Note that $\alpha$ is well-defined for all $\theta \in (0,\pi]$. When $\theta=0$, we just need to take the limit to obtain $\alpha=-1/12$. To identify the points of non-differentiability of $\widehat{D}$, it suffices to analyze the derivatives with the respect to the variables $\mathbf{Q}$ and $\mathbf{t}$. The analysis reveals that the only sources of non-differentiability occur when (type i) $\mathbf{Q}=\mathbf{Q}^\top$, $\mathbf{Q} \not= \mathbf{I}$ (i.e., at rotations of $\pi$ radians) or when (type ii) $\widehat{D}=0$. However, in both cases, the directional derivatives exist. Furthermore, $D_{\text{min},\mathcal{C}}$, as defined in \Cref{def:distance-D-hat-arbitrary-elements}, can be taken as $\widehat{D}$ when $\theta = \pi$  and $\mathbf{t} = \mathbf{0}$, which gives $D_{\text{min},\mathcal{C}} = \sqrt{2}\pi$. Thus, when $\widehat{D} < D_{\text{min},\mathcal{C}}$, it follows that $\theta < \pi$, avoiding the non-differentiable points of type i. Additionally, when $\widehat{D} > 0$, the non-differentiable points of type ii are also avoided. Therefore, the condition $0 < \widehat{D} < \sqrt{2}\pi$ guarantees that $\widehat{D}$ is differentiable, as required in \Cref{def:distance-D-hat-arbitrary-elements}.

To compute the terms of the vector field, we need to compute the derivative $\Lop_{\mathbf{V}}[\widehat{D}](\mathbf{H},\mathbf{H}_d(s^*))$. While this can be done analytically, we believe it is simpler to implement a numerical approach by evaluating the left-hand side of \eqref{eq:Leq} for $\boldsymbol{\zeta} = \mathbf{e}_i$ with a small $\epsilon=0.001$.

\subsection{The target curve}
To avoid joint limit violations, the considered curve $\mathcal{C}$, already discretized, was first designed in configuration space by:
\begin{align*}
    \mathbf{u} =& [1\ 0\ 1\ 0\ 1\ 0\ 1]^\top,\quad \mathbf{v} = [0\ 1\ 0\ 1\ 0\ 1\ 0]^\top\\
    \mathbf{q}_d(s) =& 
      \frac{\pi}{36}(\mathbf{u}+\mathbf{v}) + \pi\cos(2\pi s\slash N)\frac{\mathbf{u}}{\|\mathbf{u}\|} \\
    &
      + \frac{5\pi}{18}\bigl(\sin(2\pi s\slash N) + 1\bigr)\frac{\mathbf{v}}{2\|\mathbf{v}\|},
\end{align*}
where $N=5000$ is the number of sampled points. To compute the parameterized curve $\mathbf{H}_d(s)$ in $\text{SE}(3)$, we evaluated the forward kinematics of $\mathbf{q}_d(s)$ at each point. We then verified through simulation whether the resulting motion caused self-collisions, which led to the choice of the parameters shown above.

\subsection{The choice of S map}
Let $\boldsymbol{\xi} = [\xi_1 \ \xi_2 \ \xi_3 \ \xi_4 \ \xi_5 \ \xi_6]^\top$. The $\SL$ map used is:
\begin{align*}
    \SL[\boldsymbol{\xi}] = \begin{bmatrix}
    0 & -\xi_6 & \xi_5 & \xi_1 \\
    \xi_6 & 0 & -\xi_4 & \xi_2 \\
    -\xi_5 & \xi_4 &  0 & \xi_3 \\
     0 & 0 & 0 & 0
    \end{bmatrix}.
\end{align*}
The basis $\mathbf{E}_1, ... ,\mathbf{E}_6$ of the Lie algebra $\mathfrak{se}(3)$ can be obtained by $\mathbf{E}_k = \SL[\mathbf{e}_k]$, where $\mathbf{e}_k$ are the columns of the $6 \times 6$ identity matrix. Geometrically, the interpretation for this choice of basis is that $\boldsymbol{\xi}$ is the classical \emph{twist} in mechanics. More precisely, $\boldsymbol{\omega} \triangleq [\xi_4 \  \xi_5 \  \xi_6]^\top$ represents the $x$, $y$ and $z$ components of the angular  velocity  measured in the fixed frame, whereas $\mathbf{v} \triangleq [\xi_1 \  \xi_2 \  \xi_3]^\top$ represents the $x, y$ and $z$ velocities of the virtual point at the origin of the fixed frame, measured in this fixed frame. This is related to the object's linear velocity $\dot{\mathbf{t}}$  by $\dot{\mathbf{t}} = \boldsymbol{\omega} \times \mathbf{t} + \mathbf{v}$.

\subsection{Parameters}
\begin{figure*}[t]
    \centering
    \begin{subfigure}[b]{0.32\textwidth}
        \centering
        \includegraphics[width=\textwidth]{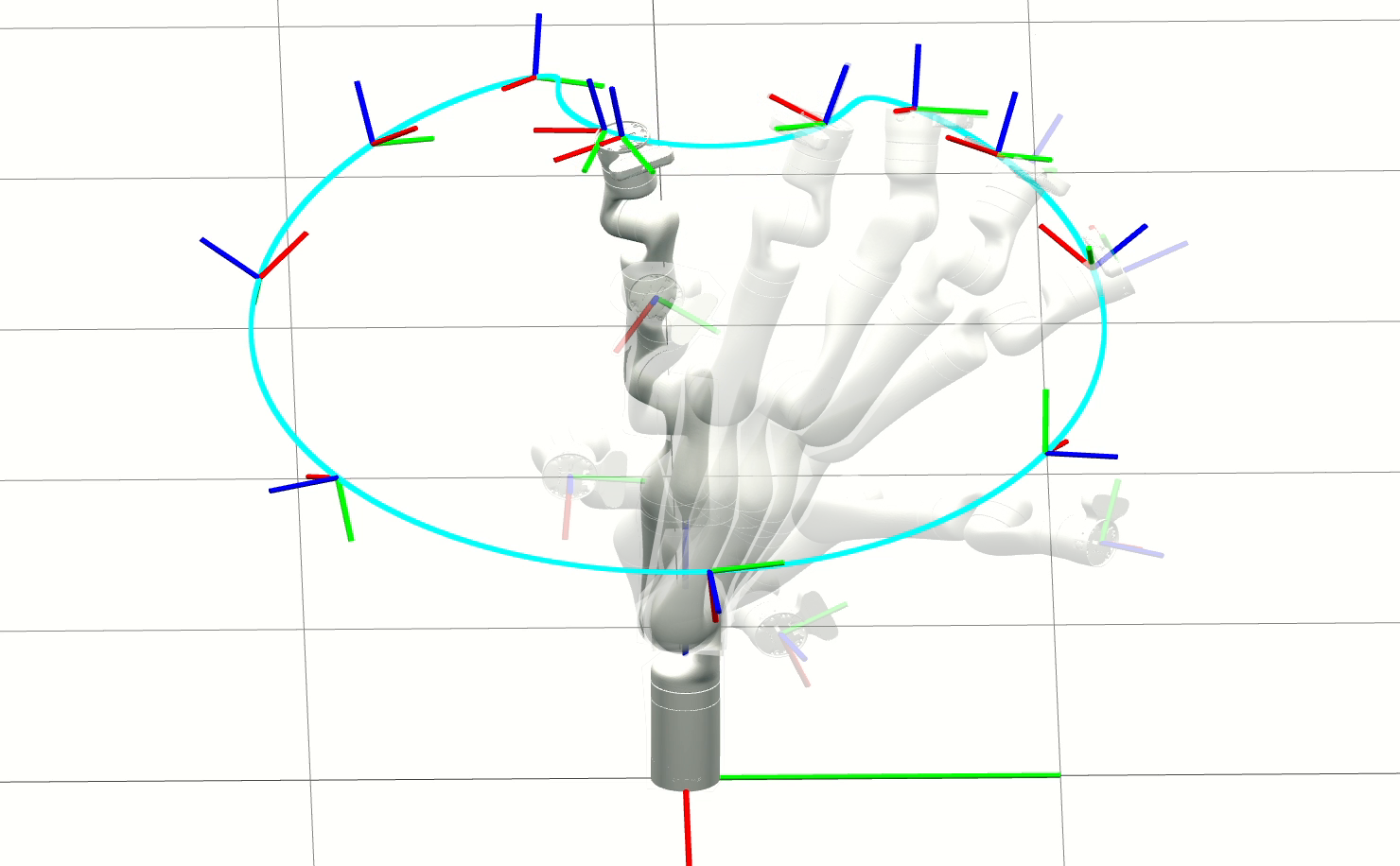} %
        \caption{$t\in[0, 50]s$}
        \label{fig:vfplot-first}
    \end{subfigure}
    \hfill
    \begin{subfigure}[b]{0.32\textwidth}
        \centering
        \includegraphics[width=\textwidth]{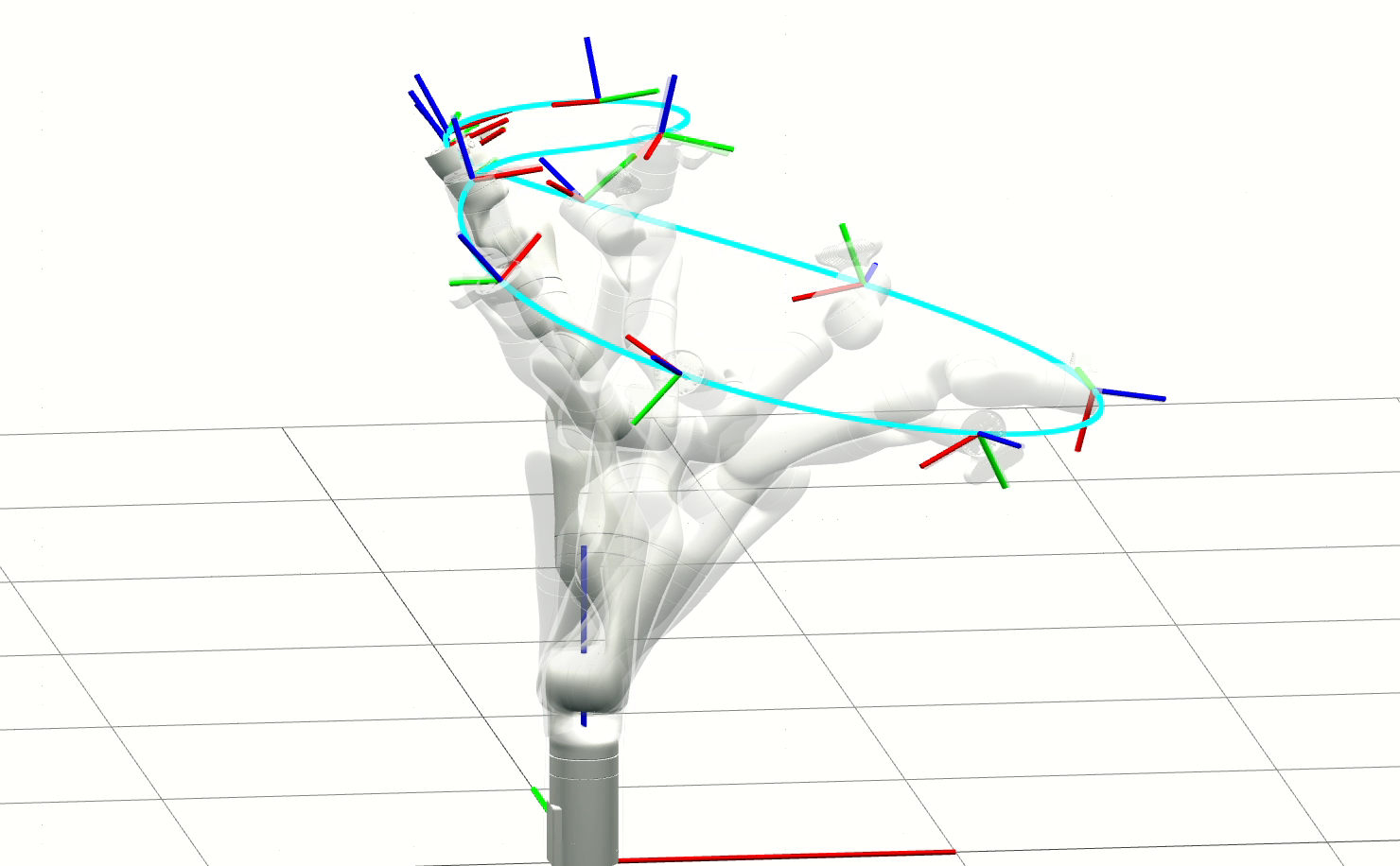} %
        \caption{$t\in[50, 100]s$}
        \label{fig:vfplot-second}
    \end{subfigure}
    \hfill
    \begin{subfigure}[b]{0.32\textwidth}
        \centering
        \includegraphics[width=\textwidth]{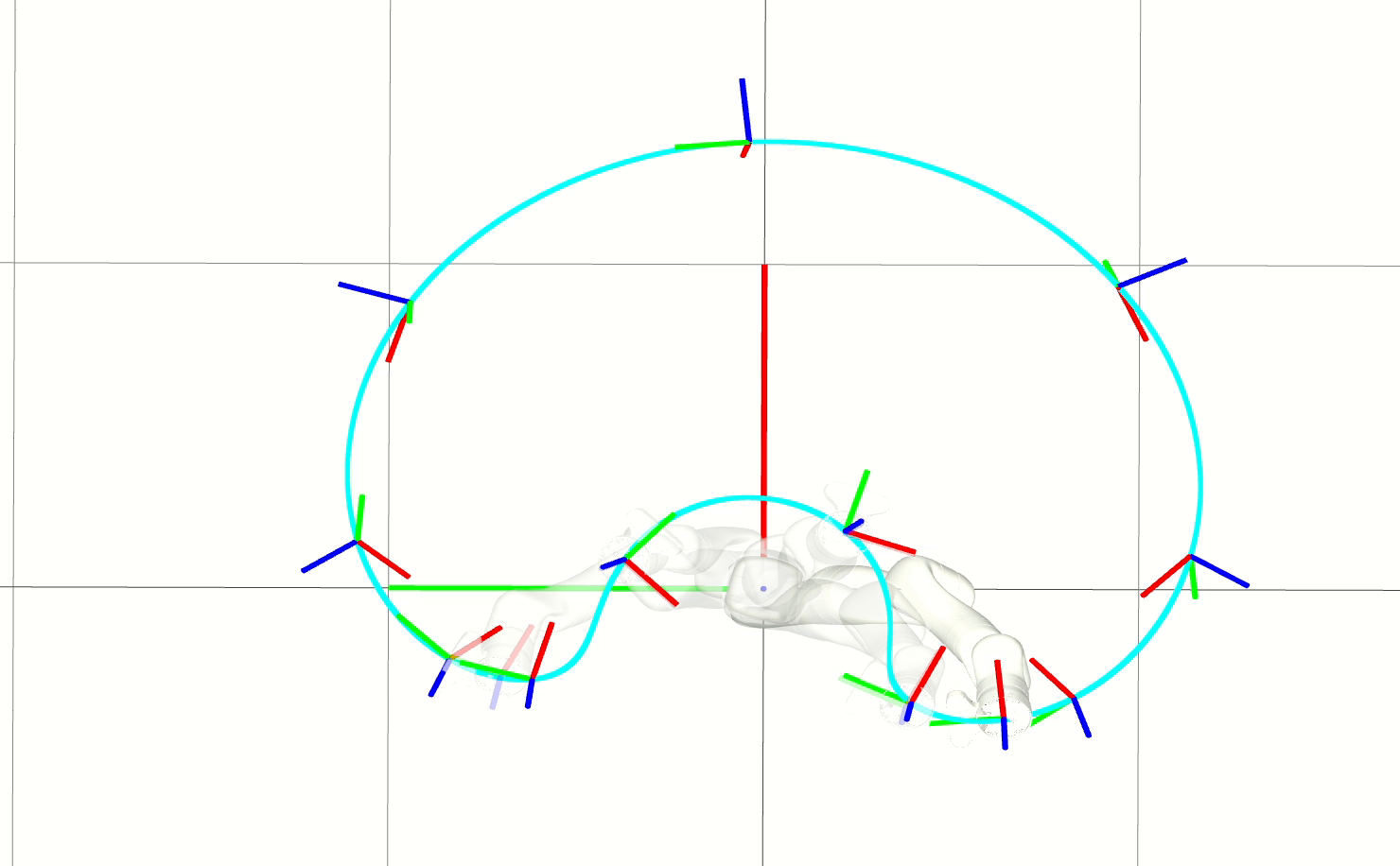} %
        \caption{$t\in[100, 150]s$}
        \label{fig:vfplot-third}
    \end{subfigure}
    \caption{The solid cyan line depicts the target curve, with target orientation frames shown by RGB axes. Intermediary configurations are shown as translucent representations, while the final configuration for each time period is opaque. The end effector orientation frame is shown by RGB axes.}
    \label{fig:vfplot-trajectory}
\end{figure*}
The optimization problem in \eqref{eq:optimization-problem-distance-definition-point-curve} was solved by determining the optimal $s=s^*$ through a brute-force approach. The curve discretization was also used to compute $\frac{d}{ds}\mathbf{H}_d(s)$ using finite differences, which is necessary for implementing $\boldsymbol{\xi}_T = \SL^{-1}(\frac{d\mathbf{H}_d}{ds}(s^*)\mathbf{H}_d(s^*)^{-1})$. The chosen gains were $k_N(D) = 0.1\tanh(0.75\sqrt{D})$ and $k_T(D) = 0.03\bigl(1 - \tanh(0.75\sqrt{D})\bigr)$. The experiment was conducted for $\qty{150}{\second}$. The initial condition $\mathbf{H}(0)$ was set by computing the forward kinematics of the initial configuration $\mathbf{q}(0) = [0 \ \frac{\pi}{18} \ 0 \ \frac{\pi}{12} \ 0 \ \frac{2\pi}{9} \ \frac{\pi}{6}]^\top$.
\subsection{Results}
\begin{figure}[t]
    \centering
    \includeinkscape[pretex=\tiny,width=\columnwidth]{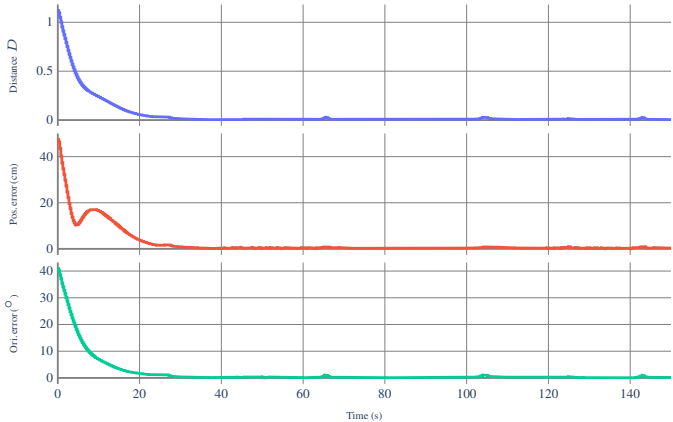}
    \caption{Value of EC-distance $D$, position error in centimeters and orientation error in degrees along time.}
    \label{fig:position-orientation-errors}
\end{figure}
We implemented the code in C++. On average, the computation of the vector field took $8.9\pm1.7$ milliseconds per iteration on a single core of an Intel i5-10300H @ 4.5 GHz CPU, with 8 GB of RAM. On average, $99.5\%$ of this time was spent solving the optimization problem in \eqref{eq:optimization-problem-distance-definition-point-curve}. Since the optimization process is highly parallelizable---allowing for the simultaneous computation of $\widehat{D}(\mathbf{H},\mathbf{H}_d(s))$ across different discretized $s$---the computational time could be significantly reduced by leveraging parallel architectures such as GPUs, SIMD, or multi-threading. 

The system's trajectory is illustrated in \Cref{fig:vfplot-trajectory}, where a 3D visualization was reconstructed from experimental data using the UAIbot
simulator\footnote{\url{https://github.com/UAIbot/UAIbotPy}}.
The values of the EC-distance function $D$ are shown in \Cref{fig:position-orientation-errors}. Additionally, \Cref{fig:position-orientation-errors} provides a more intuitive representation of the position error (in centimeters) and the orientation error (in degrees). The results confirm that the system successfully converges to the desired curve and traverses along it as expected. Once the system reaches the curve, minor deviations arise in the distance and orientation error data---evident as bumps at $65$, $105$, and $145$ seconds. This behavior stems from the challenging motion required by the curve, where the robot operates near a singular configuration and must execute frequent orientation adjustments. The experiment video is available at \href{https://youtu.be/fwRF_9K1_w0}{https://youtu.be/fwRF\_9K1\_w0}.

Although our implementation was carried out in C++, we provide a less efficient sample code in Python for clarity and accessibility, available at \url{https://github.com/fbartelt/SE3-example}.

\section{Conclusion}\label{sec:conclusion}
In this work, we presented a strategy for generating vector fields on connected matrix Lie groups that enforces convergence to and traversal along curves defined within the same group. This approach generalizes a recent work \cite{Rezende2022}, which was limited to parametric curve representations in Euclidean space. To achieve this broader applicability, we examined key properties of distance functions that maintain the desired features of the vector field approach. The three essential properties---left-invariance, chainability, and local linearity---not only support this generalization to matrix Lie groups but also permit flexibility in metric choice for Euclidean space applications. We validated our strategy through real-world robotic manipulator experiments, demonstrating successful end-effector curve tracking in $\text{SE}(3)$. Additionally, we developed an efficient algorithm for computing the vector field in this context. Finally, through the real experiment with the manipulator, we have demonstrated its applicability in real-world practical situations.

Future work will explore extensions to time-varying curves and investigate simpler distance functions with the required properties as in \Cref{thm:convergence-vector-field}. Investigating more general distance functions is also valuable, since the one presented is valid for special exponential Lie groups only. Another venue is considering underactuated systems, allowing for broader applicability of the strategy. Validation of this strategy in a real omnidirectional UAV is also of interest, along with the consideration of self-intersecting curves, which would enable more complex motions, such as lemniscate-like trajectories.

\bibliographystyle{ieeetr}
\bibliography{ref}

\vfill

\end{document}